\documentclass{article}

\PassOptionsToPackage{numbers, sort&compress}{natbib}

\usepackage[final, main]{neurips_2026}

\usepackage[utf8]{inputenc}
\usepackage[T1]{fontenc}
\usepackage{microtype}
\usepackage[table]{xcolor} 
\usepackage{enumitem}
\usepackage{booktabs}

\usepackage{amsmath, amssymb, amsthm, mathtools}
\usepackage{cancel}

\theoremstyle{plain}
\newtheorem{theorem}{Theorem}[section]
\newtheorem{proposition}[theorem]{Proposition}
\newtheorem{lemma}[theorem]{Lemma}
\newtheorem{corollary}[theorem]{Corollary}

\theoremstyle{definition}

\theoremstyle{remark}

\usepackage{graphicx}
\usepackage{booktabs}
\usepackage{subcaption}
\usepackage{wrapfig}
\usepackage{float}
\usepackage{multirow}
\usepackage{booktabs}
\usepackage{longtable}

\definecolor{codeblue}{rgb}{0.8,0.2,0.2}
\definecolor{codegray}{rgb}{0.5,0.5,0.5}
\definecolor{stdgray}{gray}{0.45}
\definecolor{best}{RGB}{0,0,0}

\definecolor{AccentTeal}{HTML}{2B6E73}
\definecolor{BgTeal}{HTML}{F3F7F7}
\definecolor{StrongBlue}{HTML}{1F4E79}
\definecolor{StrongBg}{HTML}{EEF4FA}

\usepackage{listings}
\definecolor{aquamarineblue}{HTML}{71D9E2}
\definecolor{darkgreen}{RGB}{0,80,0}
\PassOptionsToPackage{
    colorlinks=true,
    citecolor=aquamarineblue,
    linkcolor=red!70!orange,
    urlcolor=red!70!orange,
}{hyperref}
\definecolor{bestorange}{HTML}{E69F00}

\newcommand{\std}[1]{\textcolor{gray}{\scriptstyle #1}}

\usepackage[most]{tcolorbox}

\newtcolorbox{callout}[1][]{
    breakable, enhanced, colback=BgTeal, colframe=AccentTeal,
    boxrule=0.5pt, arc=2mm, left=1.5mm, right=1.5mm, top=1mm, bottom=1mm, #1
}

\newtcolorbox{calloutimportant}[1][]{
    breakable, enhanced, colback=StrongBg, colframe=StrongBlue,
    boxrule=0.6pt, arc=2mm, left=1.8mm, right=1.5mm, top=1mm, bottom=1mm,
    borderline west={2.2pt}{0pt}{StrongBlue}, #1
}

\newtcolorbox{calloutmotivation}[1][]{
    breakable, enhanced, colback=StrongBg, colframe=StrongBlue,
    boxrule=0.6pt, arc=2mm, left=1.8mm, right=1.5mm, top=1mm, bottom=1mm,
    borderline west={2.2pt}{0pt}{StrongBlue}, fonttitle=\bfseries, #1
}

\usepackage{hyperref} 
\hypersetup{
    pdftitle={Convex Compositional Reasoning Models},
    pdfauthor={Meir Roketlishvili et al.}
}

\usepackage[capitalize,noabbrev]{cleveref}
\usepackage{makecell}

\title{Convex Compositional Reasoning Models}

\author{
    \makebox[0.9\textwidth][c]{
        \begin{tabular}{c}
            Meir Roketlishvili$^{1}$ \quad Semyon Semenov$^{1}$ \quad Maksim Bobrin$^{2, 3}$ \quad Viktor Kovalchuk$^{1}$ \quad Albert Baichorov$^{1}$ \\[1.5ex]
            Abduragim Shtanchaev$^{1}$ \quad Fakhri Karray$^{1}$ \quad Dmitry V. Dylov$^{2,3}$ \quad Martin Takáč$^{1}$ \quad Arip Asadulaev$^{1}$ \\[3mm]
            \normalfont $^{1}$ Mohamed bin Zayed University of Artificial Intelligence \\
            \normalfont $^{2}$ Applied AI Institute, Computational Imaging Lab \\ \normalfont $^{3}$ AXXX \\[1.5ex]
            \texttt{Meir.Roketlishvili@mbzuai.ac.ae}
        \end{tabular}
    }
}

\begin{document}
\maketitle
\begin{abstract}
Compositional energy-based models can generalize to larger combinatorial reasoning problems by reusing a learned factor energy across many local constraints. In our paper, we show that a key bottleneck in compositional reasoning is not composition itself, but the non-convex geometry of the learned energy landscape. To solve this problem, we introduce Convex Compositional Energy Minimization (CCEM), a framework that parameterizes each factor with an input-convex neural network and optimizes the composed energy over a tight convex relaxation of the feasible set. Because convexity is preserved under summation, the global relaxed objective remains convex, enabling deterministic projected first-order optimization. CCEM is trained in two stages: factor-level contrastive learning to shape local energy basins, followed by end-to-end refinement through an unrolled projected solver. Our experiments show that our models trained on small subproblems or a single problem size transfer to larger instances without retraining.
\end{abstract}
\begin{figure}[h!]
    \centering
    \includegraphics[width=0.9\linewidth]{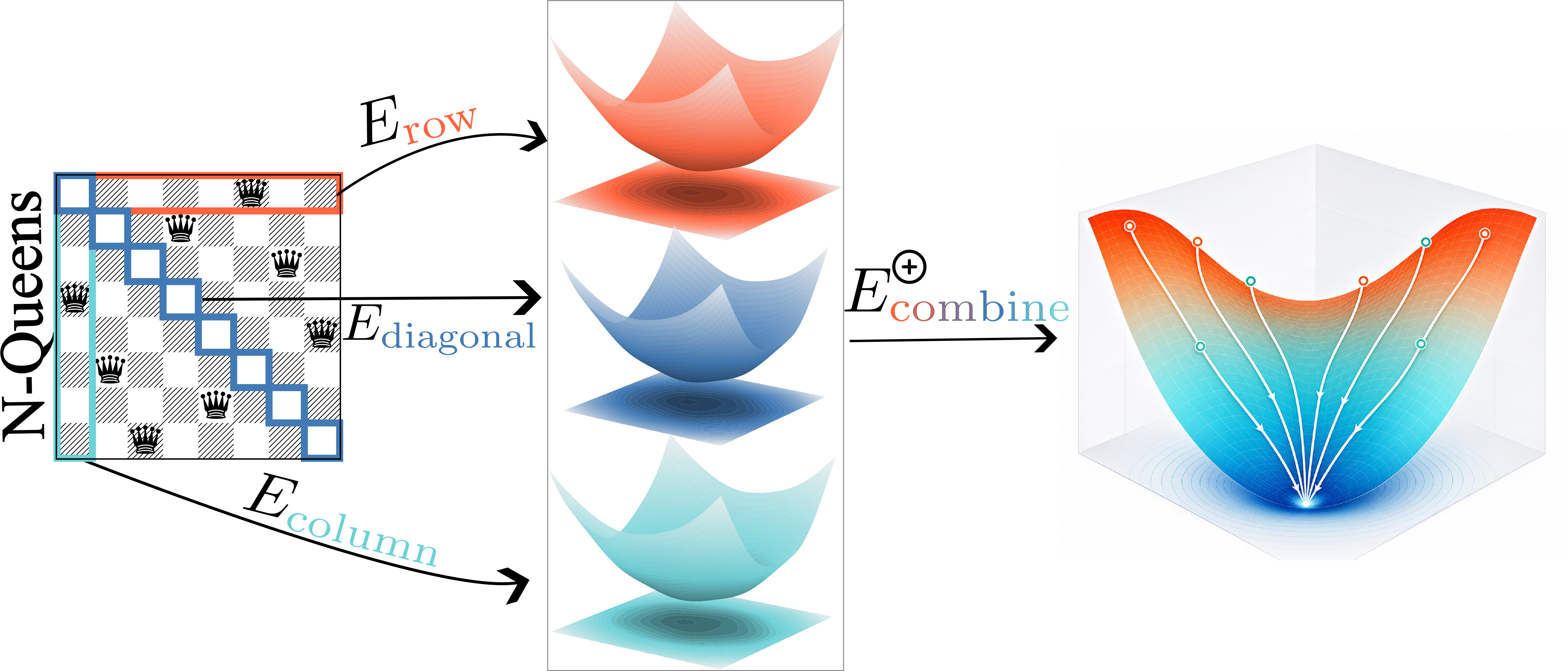}
    \caption{Our method composes convex factors into a smooth globally convex landscape, enabling efficient deterministic reasoning.}
    \vspace{-3mm}
    \label{fig:placeholder}
\end{figure}
\section{Introduction}
\label{sec:introduction}

Many reasoning problems are naturally compositional. A candidate solution is valid only when it satisfies many local constraints: every row, column, and diagonal in tabular data, every edge in graph coloring, or every clause in a satisfiability formula. This structure suggests a simple route to generalization: learn an energy model for a small local constraint, reuse it across all constraints in a larger instance, and solve the full problem by minimizing the sum of the local energies. Compositional energy-based models follow this principle. Given an instance $x$ and a candidate relaxed solution $y$, they define $E_\theta(x,y)=\sum_{k=1}^K w_k f_\theta(y_{S_k};c_k)$, where each factor $f_\theta$ scores a local scope $S_k$ under context $c_k$. The same factor network can therefore be trained on small subproblems and applied to larger instances simply by adding more terms to the sum. 

However, composition also creates the main optimization difficulty. If the factor energy is an unconstrained neural network, then the composed objective is generally non-convex. Even when each local factor appears well behaved in isolation, their sum can introduce \textbf{spurious local minima} (see proposition \ref{prop:nonconvex-spurious-minima}). \emph{As the number of constraints grows, the optimizer must search a larger and more rugged energy landscape.} Prior work addresses this issue with powerful sampling procedures, such as long diffusion chains or Parallel Energy Minimization (PEM), which use many particles and noise injection to escape poor local basins \citep{oarga2025generalizable}. These methods can be effective, but they treat the symptoms of a non-convex composed landscape rather than the cause.

In this paper, we take a different approach. \emph{We argue that the difficulty is not inherent to compositional reasoning, but to the parameterization of the local energy factors.} If each factor is convex in the decision variable, and the factors are composed using nonnegative weights, then \textbf{the full relaxed energy remains convex}. Adding more constraints may make the optimization problem larger or more ill-conditioned, but it cannot create new nonglobal local minima. This changes the role of inference: instead of relying on particle-based exploration to escape a rugged landscape, we can minimize a globally well-behaved relaxed objective with projected first-order methods (Theorem~\ref{thm:main-first-order}).

We introduce \emph{Convex Compositional Energy Minimization} (CCEM), a framework for compositional \textbf{reasoning with convex factor energies}. Each local factor is parameterized by an input-convex or partially input-convex neural network (ICNN) \citep{amos2017input}, so that it remains convex in the candidate solution variables while still conditioning flexibly on the problem context. At test time, the global energy is obtained by summing all local factors and is optimized over a tight convex relaxation of the feasible set, such as the Birkhoff polytope for $N$-Queens or simplex-based relaxations for coloring. The resulting inference procedure is deterministic, differentiable, and compatible with other solvers.

Our training pipeline has two stages. First, we train the factor energy using local contrastive supervision, assigning low energy to locally valid patterns and higher energy to invalid alternatives. Second, we refine the composed energy end-to-end by backpropagating through an unrolled projected solver. This refinement aligns the energy with the behavior of the actual inference procedure while preserving the convex structure of the model. Unlike diffusion-based compositional EBMs, CCEM does not need to learn a denoising trajectory in ambient Euclidean space; the optimization is performed directly on the convex relaxation where feasibility is enforced by projection.

The central benefit of convexity is theoretical as well as practical. We show that interactions among unrestricted non-convex factors can leave stable but globally incorrect minima in the composed objective. In contrast, convex composition rules out nonglobal local minima over the relaxation. We explicitly separate this continuous guarantee from discrete recovery: integrality and valid decoding are not consequences of convexity alone. A decoder-aware energy margin gives a sufficient condition under which finite-accuracy relaxed optimization transfers to a valid discrete answer. This condition is instance-dependent, and we make no worst-case polynomial-time claim for the NP-hard tasks considered here. We make the following \textbf{contributions}:
\begin{itemize}
    \item We identify interactions among non-convex factor geometries as a source of stable, globally incorrect minima in compositional EBMs, separating this issue from the choice of sampler.

    \item We propose a novel compositional reasoning framework that uses ICNN factor energies so that nonnegative factor summation preserves convexity of the relaxed global objective.

    \item We develop a projected optimization pipeline for training and inference directly on tight convex relaxations, replacing diffusion sampling with deterministic first-order minimization.

    \item We distinguish relaxed optimization from discrete recovery and give decoder-aware margin conditions that combine approximation, optimization, and decoding errors, together with the corresponding complexity boundary.

    \item We evaluate the method on $N$-Queens, graph coloring, and 3-SAT, reasoning benchmarks, showing that convex compositional energies can match or improve particle-based compositional EBMs while using a simpler inference procedure.
\end{itemize}

\section{Preliminaries}
\subsection{Reasoning as compositional energy minimization}

A useful way to view reasoning is as search over possible answers. Given a problem instance $x$, the model defines an energy function $E_\theta(x,y)$ over candidate solutions $y$, where valid solutions should have low energy and invalid ones higher energy. Solving the problem then means finding
\[
    \hat y = \arg\min_y E_\theta(x,y).
\]
This differs from the usual end-to-end reasoning setup, where a model directly maps inputs to outputs and often learns heuristics tied to the training distribution. In the energy-based view, inference is an optimization process, so harder problems can be given more computation at test time. Compositional reasoning adds the idea that many reasoning problems are made of smaller constraints. For example, N-Queens combines row, column, and diagonal constraints; graph coloring combines edge constraints; and SAT combines clauses. Instead of learning a separate model for each full problem size, we learn a local energy for one constraint and reuse it across the larger instance. Thus, regular reasoning treats the problem more as one object, while compositional reasoning builds the solution by combining reusable pieces. Following \citet{oarga2025generalizable}, if feasibility factorizes into $K$ local constraints, we define the global energy as
\begin{align}
y \in \mathcal{S}(x)
&\iff
\bigwedge_{k=1}^{K} \mathcal{C}_k(y_{\mathrm{sc}(k)}, c_k),
&
E_\theta(x,y)
&=
\sum_{k=1}^{K} w_k\,
f_\theta\!\left(y_{\mathrm{sc}(k)};\,c_k\right).
\label{eq:total-energy}
\end{align}
Here, $\mathrm{sc}(k)$ are the variables involved in the $k$-th constraint, $c_k$ is its context, and the same factor network $f_\theta$ is reused for all factors. In this sense, the model learns the local rule once and applies it many times. Larger problems are handled by adding more energy terms, rather than retraining.

The benefit is that learning is separated from problem size. The difficulty is that summing many learned factors can make the energy landscape harder to optimize, especially when the factors are unconstrained neural networks. This is why prior work relies on sampling methods such as Parallel Energy Minimization, while our approach focuses on making the composed landscape better behaved by choosing factor energies whose structure is preserved under summation.
\subsection{Input Convex Neural Networks.}
ICNNs \cite{amos2017input} parameterize a scalar function that is convex in (a subset of) its inputs by combining non-negative weights with non-decreasing convex activations. They have been used for control \cite{chen2018optimal}, optimal transport \cite{makkuva2020optimal,korotin2021neural}, and domain adaptation \cite{asadulaev2021connecting, pmlr-v232-asadulaev23a}. To our knowledge, ICNNs have not previously been used as the per-factor energy of a compositional reasoning model: prior compositional EBMs deliberately use unrestricted MLP scores so that the composed landscape can be \emph{shaped} by additional losses, whereas we exploit the closure of convexity under non-negative sums to keep the composed objective tractable by construction.

\section{Convex Compositional Energy Minimization}
\label{sec:method}

Our main claim is that an important obstruction in unconstrained
compositional EBMs is the geometry of the composed objective, rather than only
the choice of sampler. When local factors are unrestricted and non-convex,
minima that are legitimate for individual factors can remain stable but become
globally incorrect after the factors interact in the sum. We do not claim that
summation creates non-convexity from convex factors; the counterexample instead
shows how non-convex factor geometries can produce a bad composed landscape.

\begin{proposition}[Stable invalid minima from unrestricted factors]
\label{prop:nonconvex-spurious-minima}
There exist a convex relaxation $\mathcal Y\subset\mathbb R^d$, smooth
non-convex factors $f_k$, and a composed energy \eqref{eq:total-energy} such that $E$ has a global minimizer $y^\star\in\mathcal Y$ and a distinct
strict local minimizer $\bar y\in\mathcal Y$ with
\[
    E(y^\star)=\min_{y\in\mathcal Y}E(y),
    \qquad
    \nabla E(\bar y)=0,
    \qquad
    \nabla^2 E(\bar y)\succ 0,
    \qquad
    E(\bar y)>E(y^\star).
\]
Moreover, for some stepsize $\eta>0$, gradient descent
$y_{t+1}=y_t-\eta\nabla E(y_t)$ converges to $\bar y$ from all initializations
in an open neighborhood of $\bar y$.
\end{proposition}

Proposition~\ref{prop:nonconvex-spurious-minima} motivates imposing convexity
at the factor level. Since nonnegative sums of convex functions remain convex,
factor composition cannot introduce nonglobal local minima over a convex
relaxation. This statement concerns the continuous optimization problem. It
does not by itself imply that a relaxed minimizer is integral or that its
decoded discrete assignment is valid.

For instance $x$, let $\mathcal Y_x$ be the compact convex relaxation,
$\mathcal Z_x$ the discrete output space, $\mathcal V_x\subseteq\mathcal Z_x$
the valid discrete solutions, and $D_x:\mathcal Y_x\to\mathcal Z_x$ the
test-time decoder. Define the decoder-valid region and its robust complement by
\begin{equation}
    \mathcal A_x=\{y\in\mathcal Y_x:D_x(y)\in\mathcal V_x\},
    \qquad
    \mathcal B_x=\overline{\mathcal Y_x\setminus\mathcal A_x}.
    \label{eq:decoder-regions}
\end{equation}
The closure places argmax and threshold tie boundaries in $\mathcal B_x$, since
decoding there is not robust. For the learned energy, write
\begin{equation}
    E_x^\star=\min_{y\in\mathcal Y_x}E_\theta(x,y),\qquad
    \mathcal M_x=\arg\min_{y\in\mathcal Y_x}E_\theta(x,y),
\end{equation}
and define its decoder-aware margin
\begin{equation}
    \Gamma_x^D(E_\theta)
    =\min_{y\in\mathcal B_x}\bigl(E_\theta(x,y)-E_x^\star\bigr).
    \label{eq:decoder-margin}
\end{equation}

\begin{lemma}[Separation of relaxed and discrete guarantees]
\label{lem:decoder-separation}
If $E_\theta(x,\cdot)$ is continuous and convex on compact $\mathcal Y_x$,
then $\Gamma_x^D(E_\theta)>0$ if and only if
$\mathcal M_x\cap\mathcal B_x=\varnothing$. Moreover, every $\widehat y$ with
\[
    E_\theta(x,\widehat y)-E_x^\star
    <\Gamma_x^D(E_\theta)
\]
decodes validly: $D_x(\widehat y)\in\mathcal V_x$.
\end{lemma}

Thus integrality is sufficient but not necessary: a fractional minimizer is
harmless when it lies strictly inside a valid decoder cell. Conversely, a
globally optimized convex relaxation can still fail if its minimizer lies in an
invalid cell or on a decoding boundary. The following conditional transfer
result also includes model approximation and finite solver error.

\begin{theorem}[Decoder-aware certificate transfer]
\label{thm:decoder-certificate-transfer}
Let $G_x$ be a continuous convex reference energy whose decoder-aware margin
over $\mathcal B_x$ is
\[
    \gamma_{\mathrm{dec}}
    =\min_{y\in\mathcal B_x}G_x(y)
     -\min_{y\in\mathcal Y_x}G_x(y)>0.
\]
Suppose
$\sup_{y\in\mathcal Y_x}|E_\theta(x,y)-G_x(y)|\leq\Delta$, and the solver
returns $\widehat y$ satisfying
\[
    E_\theta(x,\widehat y)-
    \min_{y\in\mathcal Y_x}E_\theta(x,y)\leq\delta.
\]
If $2\Delta+\delta<\gamma_{\mathrm{dec}}$, then
$D_x(\widehat y)\in\mathcal V_x$.
\end{theorem}

For the compositional model, factorwise errors $\varepsilon_k$ imply
$\Delta\leq\sum_k w_k\varepsilon_k$. Theorem~\ref{thm:decoder-certificate-transfer}
is deliberately conditional: it does not establish that efficiently
constructible certificates with positive margins exist for arbitrary
$N$-Queens, graph-coloring, or 3-SAT instances. Indeed, a uniform
inverse-polynomial-margin construction together with polynomial-time convex
optimization would yield a polynomial-time algorithm for the corresponding
NP-hard search problem. Appendix~\ref{sec:decoder-proofs} gives the proofs,
this complexity boundary, and directly checkable margin certificates. For row-wise argmax, a valid unique minimizer $y^\star$ with minimum top-two
gap
\begin{equation}
    \Delta_{\mathrm{arg}} =\min_i\left(y^\star_{i,j_i^\star} -\max_{j\neq j_i^\star}y^\star_{i,j}\right)>0
\end{equation}
has a neighborhood of radius $\Delta_{\mathrm{arg}}/2$ that decodes
identically. If the energy additionally has quadratic growth
$E_\theta(x,y)-E_x^\star\geq\frac{\mu}{2}\|y-y^\star\|_2^2$, then
$E_\theta(x,\widehat y)-E_x^\star<\mu\Delta_{\mathrm{arg}}^2/8$ guarantees the
same valid decoding. This separates optimizer accuracy from decoder
robustness.

Once the distinction above is explicit, standard convex analysis gives a
useful training criterion: a first-order condition at a target relaxed solution
certifies global optimality of that target for the relaxation.

\begin{theorem}[First-order supervision is sufficient]
    \label{thm:main-first-order}
    Let $\mathcal Y\subset\mathbb R^d$ be convex, and let
    $E_\theta(x,\cdot)$ be convex and differentiable. If
    $y^\star\in\mathcal Y$ satisfies
    \[
        \left\langle
        \nabla_y E_\theta(x,y^\star),\, y-y^\star
        \right\rangle \ge 0
        \qquad
        \forall y\in\mathcal Y,
    \]
    then
    \[
        y^\star\in\arg\min_{y\in\mathcal Y}E_\theta(x,y).
    \]
\end{theorem}

The result is the first-order convexity inequality; its proof and the standard
projected-gradient convergence rates are included in
Appendix~\ref{sec:proofs}. Auxiliary energy-shaping losses can still improve
training empirically, as in Eq.~\eqref{eq:board-loss}, but they are not needed
to turn a satisfied first-order condition into global optimality of the relaxed
convex objective. Discrete correctness still requires the additional decoder
compatibility captured by Lemma~\ref{lem:decoder-separation}.

The assumed first-order condition therefore implies
$E_\theta(x,y)\ge E_\theta(x,y^\star)$ for every feasible $y$. In contrast,
for a non-convex energy, the same local condition would certify only
stationarity or local optimality, not global optimality. This gives the main theoretical reason why convex compositional EBMs retain
the benefits of energy-based inference while avoiding auxiliary
landscape-shaping losses. The model still learns an energy over candidate
solutions and performs inference by minimizing that energy, but convexity makes
the relaxed inference problem globally well behaved.

\begin{corollary}[Convergence of convex inference]
    \label{cor:main-convergence}
    Assume $E_\theta(x,\cdot)$ is convex and $L$-smooth on $\mathcal Y$, and let
    $y^\star\in\arg\min_{y\in\mathcal Y}E_\theta(x,y)$. Projected gradient descent
    with step size $1/L$,
    \[
        y_{t+1}
        =
        \Pi_{\mathcal Y}
        \left(
        y_t-\frac{1}{L}\nabla_yE_\theta(x,y_t)
        \right),
    \]
    satisfies
    \[
        E_\theta(x,y_T)-E_\theta(x,y^\star)
        \le
        \frac{L\|y_0-y^\star\|_2^2}{2T}.
    \]
    If $E_\theta(x,\cdot)$ is additionally $\mu$-strongly convex, then
    \[
        \|y_T-y^\star\|_2^2
        \le
        \left(1-\frac{\mu}{L}\right)^T
        \|y_0-y^\star\|_2^2 .
    \]
\end{corollary}

Together, Theorem~\ref{thm:main-first-order} and Corollary~\ref{cor:main-convergence} show that convex compositional energies change the role of training. Instead of learning an unconstrained landscape and then correcting it with auxiliary energy-shaping losses, the model learns inside a function class whose geometry makes local optimality certificates global. Additional regularization may still be useful empirically, but it is not required to rule out spurious minima in the relaxed composed energy.

\section{Algorithm}

In our algorithm, we use \emph{Partially Input Convex Neural Network}, so that $f_\theta(y;c)$ is convex in the scope
variable $y$ for every context $c$. Instead of training
$\nabla_y E_\theta$ to match Gaussian noise through a denoising-diffusion
objective and sampling with PEM in $\mathbb R^n$, we define $E_\theta$
directly on the convex hull of the feasible set---the simplex/Birkhoff
projection for $N$-Queens and the corresponding simplex relaxation for graph
coloring---and minimize it with a projected first-order solver. The convex
factor energy has no nonglobal local minima in $\mathcal Y$, so explicit
particle resampling is not needed to escape local basins. We instantiate
$f_\theta$ as a PICNN, which interleaves a non-convex context
pathway with a convex scope pathway:
\begin{equation}
\begin{aligned}
    u_{\ell+1}
    &=
    \sigma\!\bigl(\widetilde W^u_\ell u_\ell+\widetilde b^u_\ell\bigr),
    \\
    z_{\ell+1}
    &=
    \sigma\!\Bigl(
        W^z_\ell
        \bigl(z_\ell\odot\mathrm{softplus}(W^{zu}_\ell u_\ell)\bigr)
        +
        W^y_\ell
        \bigl(y\odot W^{yu}_\ell u_\ell\bigr)
        +
        W^u_\ell u_\ell
        +
        b_\ell
    \Bigr).
\end{aligned}
\label{eq:picnn-layer}
\end{equation}
Here $\sigma$ is a non-decreasing convex activation, $W^z_\ell$ has
non-negative entries, implemented as
$W^z_\ell=\mathrm{softplus}(\widetilde W^z_\ell)$, and the gates
$\mathrm{softplus}(W^{zu}_\ell u_\ell)$ keep the $z$-pathway non-negative.
Therefore $f_\theta(\cdot;u)$ is convex in $y$ by composition. The context encodes all factor-specific information.

\textbf{Factor-level contrastive pretraining.}
Oarga and Du \cite{oarga2025generalizable} shape the energy with a
diffusion-MSE loss and a noise-corrupted contrastive loss. Because our factor
scope is small and the convex hull is known, \textbf{we can drop the diffusion regression}
and use a purely contrastive InfoNCE objective on clean samples:
\begin{equation}
    \mathcal L_{\mathrm{NCE}}(\theta)
    =
    -\log
    \frac{
        \exp\bigl(-f_\theta(y^+;c)\bigr)
    }{
        \exp\bigl(-f_\theta(y^+;c)\bigr)
        +
        \sum_{j=1}^{J}\exp\bigl(-f_\theta(y^-_j;c)\bigr)
    }
    +
    \lambda_{\mathrm e}f_\theta(y^+;c).
    \label{eq:nce}
\end{equation}
Positives are vertices of the convex hull, such as one-hot row, column, or box
vectors and zero or one-hot diagonal vectors. Negatives are sampled outside the
hull using factor-specific proposals: under-satisfied, overloaded, two-peak,
and renormalized mixtures for $N$-Queens. This plays the role of the contrastive shaping loss in prior work, but acts at factor granularity before global composition.

\textbf{Composed-energy refinement through an unrolled solver.}
Oarga and Du \cite{oarga2025generalizable} refine the composed score by
regressing $\nabla_y\sum_k E^k_\theta$. We instead refine $E_\theta$ by
backpropagating through a differentiable projected solver on the relaxation.
Given ground-truth solutions $\{x_i^\star\}$, we run $T$ projected updates:
\begin{equation}
    y_{t+1}
    =
    \Pi_{\mathcal Y}
    \bigl(
        y_t-\eta_t A_t\nabla_yE_\theta(x,y_t)
    \bigr),
    \qquad
    t=0,\dots,T-1,
    \label{eq:projected-solver}
\end{equation}
where $A_t$ is either the identity, giving PGD, or the Adam preconditioner,
giving projected Adam. The step size $\eta_t$ follows a cosine schedule from
$\eta_0$ to $\eta_T$. We add a small uniform tie-breaking field
$\xi\sim\mathcal U([-1,1])^n$ scaled by $\lambda_{\mathrm{tb}}$ to break
permutation symmetry between equal-energy minima. This affine perturbation
preserves convexity and selects a representative from a degenerate relaxed
minimizer set; it does not restore a multimodal energy landscape or guarantee
discrete validity. With $S$ random feasible
warm-starts, we score each candidate under \texttt{no\_grad}, pick the
lowest-energy candidate, and backpropagate only through the solver trajectory
from that initialization. The refinement loss combines regression, an energy-margin term, and
hard-negative mining over the solver trajectory:
\begin{equation}
\begin{aligned}
    \mathcal L_{\mathrm{board}}
    ={}&
    \alpha\|\hat x-x^\star\|_2^2  +
    \beta\,\mathrm{softplus}
    \bigl(
        E_\theta(x^\star)-E_\theta(\hat x)+\rho
    \bigr)
    \\
    &+
    \gamma\,
    \frac{1}{|\mathcal H|}
    \sum_{\tilde x\in\mathcal H}
    \mathrm{softplus}
    \bigl(
        E_\theta(x^\star)-E_\theta(\tilde x)+\rho_{\mathrm h}
    \bigr).
\end{aligned}
    \label{eq:board-loss}
\end{equation}
Here $\mathcal H$ contains the last $H$ solver iterates, the warm-start
candidates, and a random convex mixture between the warm-start and the target.
The hard-negative term plays the role of a board-level contrastive loss and
prevents $E_\theta$ from collapsing onto trivial solutions of the regression
term.

\textbf{Inference}. At test time, we run \eqref{eq:projected-solver} from $S$
random feasible starts, score each final point by $E_\theta$, and select the
lowest-energy candidate. The task-specific decoder is row-wise argmax for
$N$-Queens, blockwise argmax for graph coloring, and thresholding for 3-SAT.
$S=1$ is the genuine single-trajectory convex-inference mode. Multiple starts
are an optional finite-budget device: a convex objective can have a non-unique
or fractional minimizer set, and truncated trajectories can approach
near-equal-energy relaxed points that decode differently. Starts therefore
improve relaxed-to-discrete recovery; they do not serve to escape nonglobal
local minima. For comparison, \emph{Parallel Energy Minimization} (PEM)
\cite[Alg.~1]{oarga2025generalizable} can be used by replacing
\eqref{eq:projected-solver} with a $P$-particle ensemble that resamples using
$w_i^{(t)}\propto\exp\bigl(-E_\theta(y_i^{(t)})\bigr)$ and re-injects scheduled Gaussian noise. PEM is the only inference mode in
which our solver becomes equivalent to the sampler of
\cite{oarga2025generalizable}.

\section{Related Work}
\label{sec:related}
\textbf{Reasoning as energy minimization.}
Casting reasoning as the search for a low-energy configuration of a learned scalar field has a long history in structured prediction \cite{lecun2006tutorial} and has recently been revived as a paradigm for combinatorial problem solving. Most directly related to us, \cite{oarga2025generalizable} compose subproblem energies at inference time and introduce Parallel Energy Minimization (PEM), a sequential-Monte-Carlo-style sampler that maintains a population of particles to escape local minima of the composed landscape. Our work shares the compositional construction but removes the source of those local minima at the parameterization level: by enforcing convexity in the decision variable per factor, the composed energy stays convex regardless of how many factors are stacked.

\textbf{Neural combinatorial optimization.}
Graph Neural Networks are the dominant tool for learning-based combinatorial optimization, including direct supervised solvers for TSP and MIS \cite{joshi2019efficient}, reinforcement-learning approaches \cite{bello2016neural}, GFlowNets \cite{zhang2022generative}, discrete diffusion solvers \cite{sun2023difusco,li2024fast}, and dedicated neural SAT solvers \cite{selsam2018learning,li2022nsnet}. Compositional EBMs side-step this by re-using a small per-factor model across arbitrarily many constraints. We keep this advantage while making inference deterministic and convex.

\textbf{Iterative computation for reasoning.}
Recurrent and program-style architectures iteratively refine partial solutions \cite{graves2014neural,kaiser2015neural,schwarzschild2021can}, and neural SAT solvers refine truth assignments via belief propagation \cite{li2022nsnet} or message passing \cite{selsam2018learning}. Our projected solver in \eqref{eq:projected-solver} is itself an iterative refinement procedure, but the iteration is performed on a continuous relaxation under a learned convex objective rather than on the discrete state space, and the number of iterations can be increased at test time to handle larger instances without retraining.


\section{Experiments}
\label{sec:experiments}

\begin{wrapfigure}{r}{0.5\textwidth}
    \vspace{-5pt}
    \centering
    \includegraphics[width=\linewidth]{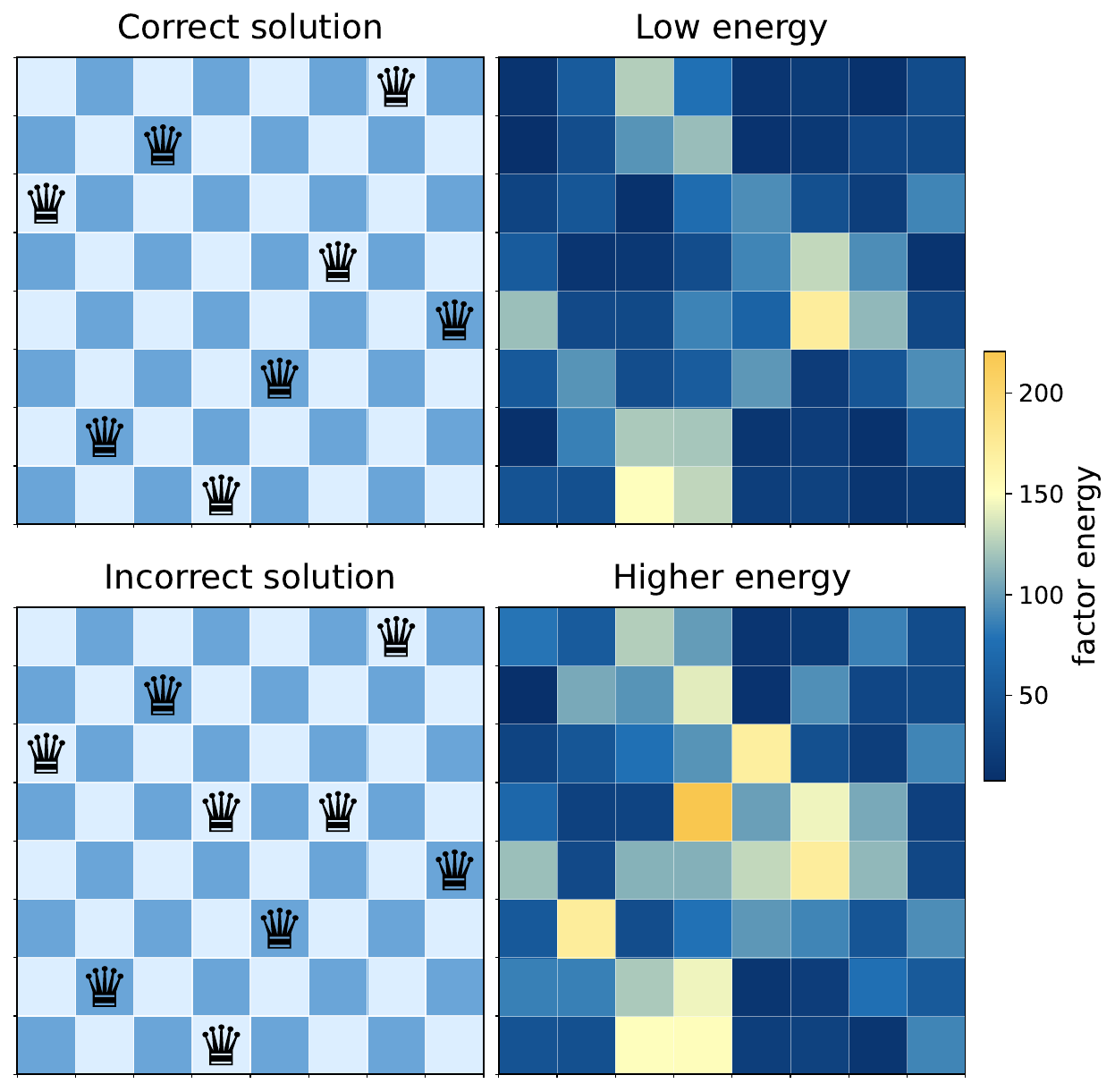}
    \caption{
        Energy map for 8-Queens. A correct solution is assigned
        low energy, while adding an extra queen produces higher local energy.
    }
    \vspace{-15pt}
    \label{fig:nqueens-energy-map}
\end{wrapfigure}

We replicate a small compositional benchmark using convex relaxations to validate the optimization stability. We train a per-constraint input-convex energy on subproblem data only, compose the full energy at test time, and run convex inference. We specifically benchmark against the original framework's $N$-Queens formulation, which decomposes the board into row, column, and diagonal energy functions trained on a single eight-queens instance, Graph Coloring formulation, which decomposes graph node‑wise and edge‑wise and the 3-SAT formulation, which composes models trained to generate valid assignments for single clauses generated in the hard phase transition region. 

We study the tradeoff between convexity strength and expressivity by moving from fully convex networks to partially convex hybrids. We hypothesize our convex formulation will match the baseline's near-perfect correct instance rates on the logical reasoning tasks while fundamentally addressing its limitations in achieving fully optimal solutions without conflicting edges in the highly constrained graph coloring environments by eliminating the local minima that plague particle-based samplers.

We address four questions: \textbf{(Q1)} does enforcing convexity in the
scope variable improve over the unconstrained MLP score network? \textbf{(Q2)} does projecting onto the exact convex hull of the feasible set replace the need for a long diffusion chain? \textbf{(Q3)} how do projected (multi-start) Adam and
PEM compare as inference samplers on the \emph{same} learned energy?
\textbf{(Q4)} does compositional generalisation to larger instances carry over to
the convex factor energy?

\subsection{Tasks and protocols}
\textbf{N-Queens.} We train on a single $8\!\times\!8$ instance, decompose it row-wise, column-wise and diagonal-wise, and re-use $f_\theta$ across rows, columns and (anti-)diagonals at  test time. We evaluate by sampling 100 boards, decoding greedily, and reporting: number of correct instances (out of 100), average number of queens placed, and average row/column/diagonal conflicts. For generalisation we additionally report results on $N\!\in\!\{4,5,6,7,9,10\}$.

\textbf{Graph coloring.} We train on a single edge-level subtask: generating a
valid coloring for a pair of adjacent nodes given the available colors. The
full graph energy is composed edge-wise. We evaluate on the COLOR benchmark
and random Erdos--Renyi, Holme--Kim, random regular expander, Paley, and complete
graphs, covering both 20--40 and 80--100 nodes. We report both the number of
conflicting edges and the exact valid-coloring rate after blockwise argmax.

\textbf{3-SAT.} Each local factor represents one clause. During factor
pretraining, complementation is used only as a local parameterization of the
single falsifying assignment. At composition and formula-level evaluation, the
original literal polarities are retained consistently; the method does not
search for one assignment that simultaneously violates every original clause
and then negate it. We train on formulas with 10--20 variables near the phase
transition, using $4.258n$ clauses, and test on 100 formulas with 20 variables
and 91 clauses. Thresholding produces the discrete assignment.
\begin{figure}[t!]
    \centering
    \includegraphics[width=\linewidth]{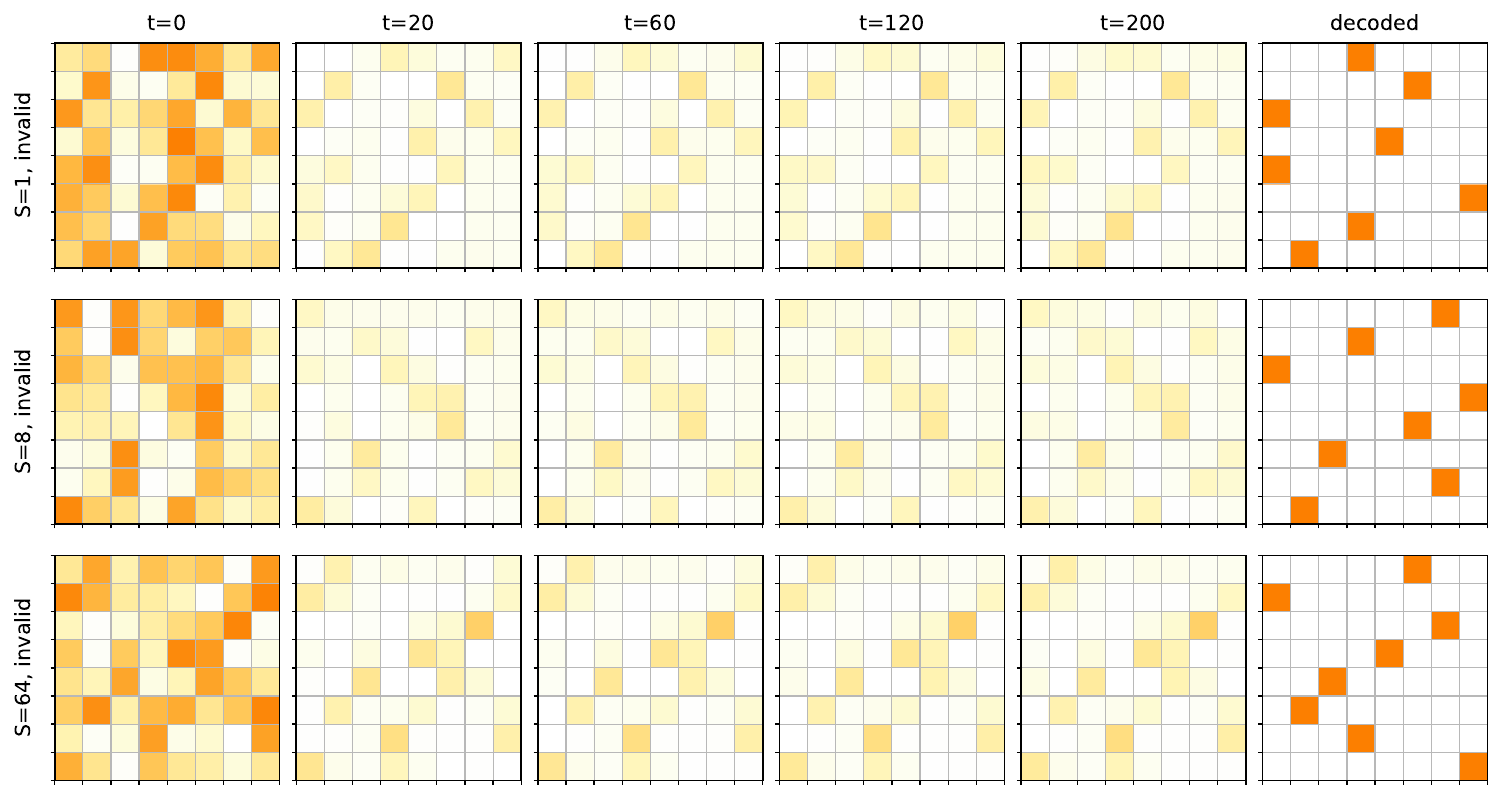}
    \caption{
        Optimized samples across projected Adam timesteps for 8-Queens.
        Intermediate columns show relaxed boards before decoding, and the final
        column shows the greedy-decoded board. Increasing the number of starts
        improves finite-budget relaxed-to-discrete recovery.
    }
    \label{fig:nqueens-timesteps}
\end{figure}

\subsection{Baselines}
For $N$-Queens we reproduce the comparison table of
\cite{oarga2025generalizable} and add our PICNN factor energy. The
non-EBM baselines (GFlowNets~\cite{zhang2022generative}, DIFUSCO~\cite{sun2023difusco},
Fast~T2T~\cite{li2024fast}) take the $N$-Queens problem encoded as a
Maximum Independent Set. The EBM baseline of
\cite{oarga2025generalizable} (\textbf{EBM-Diff}, $P{=}1024$, PEM) shares
the compositional construction with us, but uses an unconstrained score
network trained with combination of diffusion and contrastive losses. For Graph Coloring we compare against GNN-GCP~\cite{lemos2019graph}, canonical Graph Neural Networks (GCN~\cite{kipf2016semi} and GAT~\cite{velivckovic2017graph}), RL guided by Neural Algorithmic
Reasoners (XLVIN~\cite{deac2020xlvin}) and EBM baseline~\cite{oarga2025generalizable} (P=128).

\subsection{Implementation details}
Whereas \cite{oarga2025generalizable} runs $E_\theta$ on $\mathbb R^n$ and
relies on the diffusion process to deliver near-binary samples, we
optimise on the smallest tight convex set that contains every feasible
configuration, projected by alternating row/column simplex
sweeps $\mathcal{Y}=\mathcal{B}_N=\{X\in[0,1]^{N\times N}\mid
\sum_iX_{ij}=\sum_jX_{ij}=1\}$. Diagonal factors live on the relaxed simplex
$\{y\in[0,1]^d\mid\sum_i y_i\le 1\}$.
We denote the projector $\Pi_{\mathcal{Y}}$. Crucially, $\Pi_{\mathcal{Y}}$
is differentiable almost everywhere, which lets gradients flow through the
solver.

\definecolor{darkgreen}{RGB}{0,80,0}

\begin{table}[t]
\centering
\small
\caption{8-Queens evaluation (100 sampled boards). Baselines are reproduced
from \cite[Tab.~1]{oarga2025generalizable}; CCEM uses $T=50$ and $S=64$.}
\label{tab:nqueens-main}
\begin{tabular}{lccc}
\toprule
Model & Type & Correct $\uparrow$ & Queens placed $\uparrow$ \\
\midrule
LWD                              & RL+S         & 22 & $7.10\std{\pm0.57}$ \\
GFlowNets                        & UL+S         & 14 & $6.93\std{\pm0.59}$ \\
DIFUSCO ($T{=}50$)               & SL+S         & 17 & $6.94\std{\pm0.65}$ \\
Fast T2T ($T_S{=}1,T_G{=}1$)     & SL+S         & 21 & $6.82\std{\pm0.88}$ \\
Fast T2T ($T_S{=}5,T_G{=}5$, GS) & SL+GS        & 41 & $7.38\std{\pm0.54}$ \\
EBM-Diff ($P{=}1024$, PEM) & SL+PEM & 97 & $7.97\std{\pm0.17}$ \\
\midrule
\textbf{CCEM (ours)}  & SL+Adam  &  \cellcolor{darkgreen!10}{\textbf{100}}  & \cellcolor{darkgreen!10}{$\textbf{8.00}\std{\pm0.00}$} \\
\bottomrule
\end{tabular}
\end{table}

The factor PICNN has 3 hidden layers of width 256 with softplus
($N$-Queens, graph coloring) or ReLU (3-SAT). Phase~1 is trained for 1000
epochs of AdamW at learning rate $10^{-3}$, weight decay $10^{-4}$, and $J$
negatives per positive ($J=4$ for $N$-Queens, 7 for 3-SAT, and 5 for graph
coloring). Phase~2 is trained for 300 epochs with cosine step decay
$0.1\!\to\!0.02$ (train) and $0.05\!\to\!0.005$ (evaluation), $S=4$ starts in
training, and up to $S=64$ in evaluation. We use $T=50$ evaluation steps for
$N$-Queens and $T=140$ for graph coloring and 3-SAT. The hard-negative count is
$H=12$, the loss weights are $(\alpha,\beta,\gamma)=(1,0.25,0.25)$, the margins
are $\rho=\rho_{\mathrm h}=0.1$, and $\lambda_{\mathrm{tb}}=10^{-4}$. We use
simplex projection and a single GPU. Appendix~\ref{sec:compute-comparison}
provides hardware-specific timing and energy-evaluation counts.

\subsection{Results}
\begin{wrapfigure}{r}{0.45\textwidth}
\centering
\vspace{-3mm}
\caption{Phase-2 loss ablation on 8-Queens (100 sampled boards, projected Adam, $S{=}64$).}
\label{tab:loss-ablation-s64}
\begin{tabular}{cccc}
\toprule
MSE ($\alpha$) & Rank ($\beta$) & Hard ($\gamma$) & Correct \\
\midrule
\checkmark &  &  & 100 \\
 & \checkmark & \checkmark & 100 \\
\checkmark & \checkmark &  & 100 \\
\checkmark &  & \checkmark & 100 \\
\checkmark & \checkmark & \checkmark & 100 \\
\bottomrule
\vspace{-5mm}
\end{tabular}
\end{wrapfigure}

\textbf{Convexity vs.\ MLP score.}
Table~\ref{tab:nqueens-main}, Table~\ref{tab:gc-main}, and
Table~\ref{tab:sat-main} compare CCEM with compositional and task-specific
baselines. CCEM solves all sampled 8-Queens boards at its largest budget and
reduces graph-coloring conflicts across every tested distribution. The revised
3-SAT formulation improves clause satisfaction from the submitted CCEM value
$0.9665$ to $0.9959$ at $S=64$, substantially narrowing the gap to the
EBM+PEM value $0.9985$.

\textbf{Tight relaxation replaces the diffusion chain.}
\cite{oarga2025generalizable} need $T{=}100$ diffusion steps and a
$P{=}1024$ particle ensemble to lift Gaussian noise to the binary
solution manifold. With the Birkhoff projector $\Pi_{\mathcal{B}_N}$
applied at every solver step, the relaxation is already tight, and we
match their reported accuracy with $T{=}50$ projected-Adam steps and
$S{=}64$ starts (Table~\ref{tab:nqueens-main}). The EBM-Diff row is the result
reported by \citet{oarga2025generalizable}; our reproduction using the authors'
public code solved 75/100 boards.

\textbf{Compositional generalisation to larger $N$.}
We train once on $N{=}8$, evaluate on $N\!\in\!\{7,8,9,10\}$ by re-summing
$f_\theta$ over the rows, columns and diagonals of the larger board, and
sweep the inference budget. Because $f_\theta$ conditions on the scope
mask and active length, the same network generalises to all four sizes
without retraining. Increasing $S$ for projected Adam and
$P$ (particles) for PEM both monotonically improve correct-instance
counts; PEM with $P{=}128$ matches projected Adam with $S{=}32$ on
$N{=}10$. For CCEM, this gain reflects finite-step optimization and decoder
selection among near-equal relaxed candidates, not escape from nonglobal local
minima.

\definecolor{lightgreen}{rgb}{0.9, 1, 0.6}  
\definecolor{darkgreen}{RGB}{0,80,0}

\begin{table}[t]

\centering
\scriptsize
\caption{Graph Coloring evaluation (100 sampled tasks).}
\label{tab:gc-main}
\resizebox{\textwidth}{!}{%
\begin{tabular}{lcccccc}
\toprule
Distribution & GCN & GAT & XLVIN & GNN-GCP & EBM (P=128) & CCEM (Ours) \\
\midrule
Erdos Renyi & 46.80 $\pm$ 20.47 & 34.00 $\pm$ 11.55 & 25.00 $\pm$ 7.81 & 15.20 $\pm$ 4.32 & 8.60 $\pm$ 4.82 & \cellcolor{darkgreen!10}{\textbf{2.20 $\pm$ 2.14}} \\
Erdos Renyi 2 & 151.60 $\pm$ 12.09 & 130.20 $\pm$ 11.47 & 93.80 $\pm$ 31.12 & 53.80 $\pm$ 8.34 & 29.20 $\pm$ 8.05 & \cellcolor{darkgreen!10}{\textbf{4.60 $\pm$ 2.94}} \\
Holme Kim & 74.00 $\pm$ 14.74 & 51.20 $\pm$ 10.03 & 29.00 $\pm$ 7.75 & 13.20 $\pm$ 7.46 & 10.60 $\pm$ 2.70 & \cellcolor{darkgreen!10}{\textbf{5.00 $\pm$ 3.41}} \\
Holme Kim 2& 408.00 $\pm$ 26.40 & 253.20 $\pm$ 50.71 & 182.60 $\pm$ 24.73 & 55.20 $\pm$ 12.63 & 59.00 $\pm$ 3.74 & \cellcolor{darkgreen!10}{\textbf{14.00 $\pm$ 1.26}} \\
Regular Expander & 87.60 $\pm$ 22.58 & 58.60 $\pm$ 12.44 & 29.00 $\pm$ 7.75 & 15.40 $\pm$ 6.65 & 11.00 $\pm$ 4.89 & \cellcolor{darkgreen!10}{\textbf{4.20 $\pm$ 3.43}} \\
Regular Expander 2 & 144.80 $\pm$ 6.90 & 118.80 $\pm$ 12.59 & 112.60 $\pm$ 10.97 & 141.60 $\pm$ 69.47 & 37.20 $\pm$ 4.71 & \cellcolor{darkgreen!10}{\textbf{14.40 $\pm$ 3.56}} \\
Paley & 285.00 $\pm$ 117.70 & 239.20 $\pm$ 159.43 & 151.80 $\pm$ 92.13 & 91.20 $\pm$ 63.14 & 34.80 $\pm$ 20.27 & \cellcolor{darkgreen!10}{\textbf{1.80 $\pm$ 3.60}} \\
Complete & 46.00 $\pm$ 15.04 & 46.00 $\pm$ 15.04 & 34.80 $\pm$ 16.42 & 30.00 $\pm$ 2.54 & 3.40 $\pm$ 1.14 & \cellcolor{darkgreen!10}{\textbf{0.00 $\pm$ 0.00}} \\
\bottomrule
\end{tabular}}
\end{table}

The conflict-count gains translate into higher exact discrete success rates:
CCEM improves over EBM+PEM on six of the eight graph families and ties on the
two hardest zero-success cases (Table~\ref{tab:gc-valid}). In particular, it
solves 29/100 Erdos--Renyi, 59/100 Paley, and 100/100 complete-graph instances.


\textbf{Loss ablation.}
We ablate the Phase~2 loss components and observe that all variants remain
surprisingly strong. This suggests that the main gain comes not from a single
loss term, but from the compositional energy structure itself: row, column, and
diagonal factors already impose a strong inductive bias. With $S{=}64$ starts,
all variants solve all sampled instances, indicating that the inference
procedure can reliably recover good solutions once enough starts are used.

\textbf{Sharpness diagnostics.}
For every relaxed heatmap we report row entropy, the top-2 row gap, the
fraction of rows with $\max_jX_{ij}>0.9$, and row/col sum-to-one
errors after projection. The PICNN energy concentrates on a single
column per row (gap $\approx 1$, entropy $\approx 0$, sum errors
$<10^{-3}$), confirming that the convex factor energy yields a
\emph{decisive} basin on the Birkhoff polytope rather than a smeared
heatmap that requires temperature-style decoding.

\textbf{Relaxation-to-decoding diagnostics.}
To attribute failures, we separate solver accuracy, relaxation looseness, and
decoding. For a final iterate $y_T$, we use the projected residual
\[
r_{\mathrm{opt}}(y_T)=\eta^{-1}
\|y_T-\Pi_{\mathcal Y_x}(y_T-\eta\nabla E_x(y_T))\|_2,
\]
fractionality
$\phi(y_T)=1-m^{-1}\sum_{i=1}^m\max_j y_{T,ij}$, the top-two gap, and the
pre/post-decoding energy increase
$\Delta E_{\mathrm{dec}}=E_x(D_x(y_T))-E_x(y_T)$. We pair these with relaxed
and exact constraint violations. A large residual indicates finite-optimization
error; a small residual with a large relaxed-to-discrete gap indicates a loose
relaxation; and a small residual and relaxation gap but large decoder regret or
a small top-two gap indicates decoding instability. On small instances, we
compare
$E_{\mathcal Y}^\star=\min_{\mathcal Y_x}E_x$,
$E_{\mathcal Z}^\star=\min_{\mathcal Z_x}E_x$, and
$E_{\mathcal V}^\star=\min_{\mathcal V_x}E_x$ to distinguish relaxation,
model, and decoder gaps.

\section{Limitations and Broader Impacts}
\label{sec:limitations}

Compositional energy minimization adds constraints by summing energies, and input-convex architectures keep the resulting decision landscape free of new nonconvex traps. Prior work shows strong generalization, including SAT instances with hundreds of clauses after training only on single clauses, and unseen Paley or random regular expander graphs. In our view, composition turns generalization into a convex program: harder tasks mainly require more solver precision, not escape from pathological geometry. The tradeoff is that convexity can be too restrictive, producing flat regions, fractional relaxations, or poor fits to multimodal discrete solution spaces.

\section{Conclusion}
We introduced Convex Compositional Energy Minimization (CCEM), a framework for compositional reasoning that replaces unconstrained factor energies with input-convex factor models and performs inference over tight convex relaxations. By preserving convexity under nonnegative factor composition, CCEM removes a central source of spurious local minima and enables deterministic projected first-order optimization in place of long diffusion or particle-based sampling procedures. Our theory guarantees global optimization of the relaxed objective and gives additional decoder-aware conditions for discrete recovery, without claiming that convexity alone solves the combinatorial problem. Experiments on three structured reasoning tasks show transfer from small training instances to larger problems, improved graph-coloring validity, revised competitive 3-SAT performance, and substantially lower inference budgets than particle-based composition. These results suggest that controlling factor-level geometry is a promising direction for efficient neural reasoning when the relaxation and decoder are well aligned.

\bibliographystyle{plainnat}
\bibliography{bib/references}

\appendix
\newpage
\section{Proofs}
\label{sec:proofs}
\begin{proposition}[Non-convex factor composition admits spurious stable minima]
\label{prop:nonconvex-spurious-minima_a}
Let $\mathcal Y \subset \mathbb R^d$ be a convex relaxation of a discrete reasoning problem, and let a compositional energy be
\[
E(y) = \sum_{k=1}^K f_k(y_{S_k}),
\]
where each factor $f_k : \mathbb R^{|S_k|} \to \mathbb R$ is twice continuously differentiable and
$S_k \subseteq \{1,\dots,d\}$ denotes the scope of factor $k$. Assume the factors are not constrained to be convex in their scope variables.

Then there exist smooth factor energies $\{f_k\}_{k=1}^K$ and a feasible global solution
$y^\star \in \mathcal Y$ such that:
\begin{enumerate}
    \item $y^\star$ is a global minimizer of the composed energy:
    \[
    E(y^\star) = \min_{y \in \mathcal Y} E(y).
    \]

    \item The composed energy also contains a point $\bar y \in \mathcal Y$, with
    $\bar y \neq y^\star$, satisfying
    \[
    \nabla E(\bar y) = 0,
    \qquad
    \nabla^2 E(\bar y) \succ 0.
    \]

    \item Therefore, $\bar y$ is a strict local minimum of $E$, but not a global minimizer:
    \[
    E(\bar y) > E(y^\star).
    \]
\end{enumerate}

Moreover, for gradient descent
\[
y_{t+1} = y_t - \eta \nabla E(y_t),
\]
there exists a stepsize $\eta > 0$ and an open neighborhood
$U \subset \mathcal Y$ of $\bar y$ such that every initialization $y_0 \in U$
converges to $\bar y$. Hence failure from such initializations is caused by the
geometry of the composed energy landscape, rather than by the absence of a
sufficiently expressive local inference procedure.
\end{proposition}

\begin{proof}
It suffices to give a one-dimensional construction. Let
\[
\mathcal Y = [-2,2],
\]
and define two factor energies
\[
f_1(y) = (y^2 - 1)^2,
\qquad
f_2(y) = \alpha (y - 1)^2,
\]
where $\alpha > 0$ is a small constant. The composed energy is
\[
E_\alpha(y)
= f_1(y) + f_2(y)
= (y^2 - 1)^2 + \alpha (y - 1)^2.
\]

The point
\[
y^\star = 1
\]
is a global minimizer because
\[
E_\alpha(1) = 0,
\]
and both terms in $E_\alpha$ are nonnegative.

Now consider the behavior near $y=-1$. When $\alpha = 0$, we have
\[
E_0(y) = (y^2 - 1)^2,
\]
and $y=-1$ is a strict local minimum because
\[
E_0'(-1) = 0,
\qquad
E_0''(-1) = 8 > 0.
\]

Since $E_\alpha$ is a smooth perturbation of $E_0$, the implicit function
theorem implies that for all sufficiently small $\alpha > 0$, there exists a
critical point $\bar y_\alpha$ near $-1$ such that
\[
E_\alpha'(\bar y_\alpha) = 0.
\]
Furthermore, by continuity of the second derivative,
\[
E_\alpha''(\bar y_\alpha) > 0
\]
for all sufficiently small $\alpha > 0$. Therefore, $\bar y_\alpha$ is a strict
local minimum of $E_\alpha$.

However,
\[
\bar y_\alpha \neq 1.
\]
Since $E_\alpha(1)=0$ is the global minimum and $E_\alpha(y)>0$ for all
$y \neq 1$, we have
\[
E_\alpha(\bar y_\alpha) > E_\alpha(1).
\]
Thus $\bar y_\alpha$ is a strict spurious local minimum of the composed energy.

Finally, because
\[
E_\alpha''(\bar y_\alpha) > 0,
\]
there exists an open neighborhood $U$ of $\bar y_\alpha$ on which
$E_\alpha$ is strongly convex and has Lipschitz-continuous gradient. For a
sufficiently small stepsize $\eta > 0$, gradient descent initialized in $U$
remains in $U$ and converges to the unique local minimizer $\bar y_\alpha$.
This proves the claim.
\end{proof}

\subsection{Decoder-aware relaxed-to-discrete guarantees}
\label{sec:decoder-proofs}

\begin{proof}[Proof of Lemma~\ref{lem:decoder-separation}]
The set $\mathcal B_x$ is closed and is therefore compact as a subset of
compact $\mathcal Y_x$. Continuity implies that the minimum in
Eq.~\eqref{eq:decoder-margin} is attained. It equals zero exactly when a global
minimizer of $E_\theta(x,\cdot)$ belongs to $\mathcal B_x$, proving the first
claim. If $\widehat y$ decoded invalidly, then
$\widehat y\in\mathcal Y_x\setminus\mathcal A_x\subseteq\mathcal B_x$ and
therefore
\[
E_\theta(x,\widehat y)-E_x^\star\geq\Gamma_x^D(E_\theta),
\]
which contradicts the assumed strict inequality.
\end{proof}

\begin{proof}[Proof of Theorem~\ref{thm:decoder-certificate-transfer}]
Let $G^\star=\min_{\mathcal Y_x}G_x$ and
$E^\star=\min_{\mathcal Y_x}E_\theta(x,\cdot)$. Uniform approximation gives
$E^\star\leq G^\star+\Delta$. For every $y\in\mathcal B_x$,
\[
E_\theta(x,y)
\geq G_x(y)-\Delta
\geq G^\star+\gamma_{\mathrm{dec}}-\Delta
\geq E^\star+\gamma_{\mathrm{dec}}-2\Delta.
\]
If $\widehat y\in\mathcal B_x$, the left-hand side is at least
$\gamma_{\mathrm{dec}}-2\Delta$ above $E^\star$, whereas solver accuracy bounds
it by $\delta$. The condition $2\Delta+\delta<\gamma_{\mathrm{dec}}$ rules this
out. Hence $\widehat y\notin\mathcal B_x$ and its decoding is valid. For a
compositional approximation, the triangle inequality gives
$\Delta\leq\sum_k w_k\varepsilon_k$.
\end{proof}

Theorem~\ref{thm:decoder-certificate-transfer} is a transfer statement, not an
existence theorem. For example, if a valid assignment $z_x$ were already
known, $G_x(y)=\|y-z_x\|_2^2$ would have a unique valid minimizer and a positive
decoder margin. Constructing this energy requires knowing the answer, however;
similarly, describing the convex hull of all valid assignments can encode the
entire combinatorial problem.

\begin{proposition}[Complexity boundary for uniform convex certificates]
\label{prop:complexity-boundary}
Consider a family of feasible instances of an NP-hard search problem. Suppose a
polynomial-time procedure constructs from every instance $x$ a tractable convex
relaxation $\mathcal Y_x$, a convex energy $G_x$, and a polynomial-time decoder
$D_x$ with decoder-aware margin at least $1/p(|x|)$ for a polynomial $p$.
Suppose also that the convex problem can be solved in polynomial time to error
less than $1/p(|x|)$. Then the associated NP-hard search problem can be solved
in polynomial time.
\end{proposition}

\begin{proof}
Construct $(\mathcal Y_x,G_x,D_x)$, optimize $G_x$ below its guaranteed margin,
and apply $D_x$. Lemma~\ref{lem:decoder-separation} makes the decoded output a
valid solution. Thus CCEM makes no claim that uniformly constructible,
inverse-polynomial-margin certificates exist for all instances of 3-SAT or
graph coloring; its guarantees are conditional and instance dependent.
\end{proof}

For product-simplex relaxations, a useful sufficient condition can be checked
at a valid discrete vertex.

\begin{theorem}[Strict margin implies exact argmax decoding]
\label{thm:strict-argmax-margin}
Let $\mathcal Y_x=\prod_i\Delta_q$, let $z_i^\star=e_{c_i}$ be a valid one-hot
solution, and let $E_\theta(x,\cdot)$ be convex and differentiable. Define
\[
m=\min_i\min_{j\neq c_i}
\left[
\frac{\partial E_\theta(x,z^\star)}{\partial y_{ij}}
-\frac{\partial E_\theta(x,z^\star)}{\partial y_{i,c_i}}
\right].
\]
If $m>0$, then $z^\star$ is the unique relaxed minimizer. Moreover, every $y$
whose blockwise argmax differs from $z^\star$ satisfies
\[
E_\theta(x,y)-E_\theta(x,z^\star)\geq m/2.
\]
Consequently, any $\widehat y$ with optimization error below $m/2$ decodes
exactly to $z^\star$.
\end{theorem}

\begin{proof}
Let $g=\nabla_yE_\theta(x,z^\star)$. By convexity,
\[
E_\theta(x,y)-E_\theta(x,z^\star)
\geq
\sum_i\sum_{j\neq c_i}y_{ij}(g_{ij}-g_{i,c_i})
\geq m\sum_i(1-y_{i,c_i}).
\]
The bound is positive for every $y\neq z^\star$, so the minimizer is unique.
If argmax differs in block $i$, then $y_{i,c_i}\leq1/2$ and the bound is at
least $m/2$.
\end{proof}
The unknown optimization error can be upper-bounded by the Frank--Wolfe gap
\[
g_{\mathrm{FW}}(\widehat y)
=\max_{s\in\mathcal Y_x}
\langle\nabla E_\theta(x,\widehat y),\widehat y-s\rangle.
\]
Therefore $g_{\mathrm{FW}}(\widehat y)<m/2$ is a directly checkable certificate
of exact decoding under Theorem~\ref{thm:strict-argmax-margin}.

The trained-model margin in Eq.~\eqref{eq:decoder-margin} can also be bounded
directly. For the argmax and threshold decoders used in our tasks,
$\mathcal B_x$ is a finite union of polyhedral violation regions
$\mathcal C_{x,q}$. A witness $q$ is a same-color edge for graph coloring, a
conflicting queen pair for $N$-Queens, or a clause whose decoded literals are
all false for 3-SAT. Each witness is imposed by linear inequalities. Hence
\[
E_{x,\mathrm{bad}}^\star
=\min_q\min_{y\in\mathcal C_{x,q}}E_\theta(x,y)
\]
is a collection of convex problems. For any candidate
$\widetilde y_q\in\mathcal C_{x,q}$, convexity gives the lower bound
\[
L_q=E_\theta(x,\widetilde y_q)
+\min_{z\in\mathcal C_{x,q}}
\langle\nabla E_\theta(x,\widetilde y_q),z-\widetilde y_q\rangle
\leq\min_{y\in\mathcal C_{x,q}}E_\theta(x,y).
\]
For any feasible relaxed point $\widetilde y\in\mathcal Y_x$, the condition
$\min_qL_q-E_\theta(x,\widetilde y)>0$ certifies a positive decoder-aware
margin. The inner problem defining $L_q$ is linear over the original relaxation
plus decoder inequalities; a valid subgradient gives the same bound for a
nondifferentiable PICNN.

\begin{theorem}[First-order supervision suffices for convex energy minima]
\label{thm:first-order-supervision_a}
Let $\mathcal Y \subset \mathbb R^d$ be a nonempty compact convex relaxation of
the solution space, and let
\[
E_\theta(x,y)
\]
be an energy function that is convex and differentiable in $y$ for every fixed
instance $x$. Let $y^\star \in \mathcal Y$ be a correct solution for instance
$x$.

Suppose that training enforces the first-order optimality condition
\[
\left\langle
\nabla_y E_\theta(x,y^\star),\, y-y^\star
\right\rangle
\ge 0
\qquad
\text{for all } y\in\mathcal Y.
\]
Then $y^\star$ is a global minimizer of the energy:
\[
y^\star \in
\arg\min_{y\in\mathcal Y} E_\theta(x,y).
\]

Thus, for convex energy models, an additional contrastive landscape-shaping loss
is not needed to make the correct solution a global minimizer. It is sufficient
to enforce the convex first-order optimality condition at the correct relaxed
solution. Discrete validity additionally requires the decoder-aware condition in
Lemma~\ref{lem:decoder-separation}.
\end{theorem}

\begin{proof}
Fix an instance $x$. Since $E_\theta(x,\cdot)$ is convex and differentiable on
$\mathcal Y$, the first-order convexity inequality gives, for every
$y\in\mathcal Y$,
\[
E_\theta(x,y)
\ge
E_\theta(x,y^\star)
+
\left\langle
\nabla_y E_\theta(x,y^\star),\, y-y^\star
\right\rangle .
\]
By the assumed first-order optimality condition,
\[
\left\langle
\nabla_y E_\theta(x,y^\star),\, y-y^\star
\right\rangle
\ge 0
\qquad
\text{for all } y\in\mathcal Y.
\]
Therefore,
\[
E_\theta(x,y)
\ge
E_\theta(x,y^\star)
\qquad
\text{for all } y\in\mathcal Y.
\]
Hence $y^\star$ is a global minimizer:
\[
y^\star \in
\arg\min_{y\in\mathcal Y} E_\theta(x,y).
\]
\end{proof}

\begin{corollary}[Approximate first-order supervision gives approximate optimality]
\label{cor:approx-first-order_a}
Under the assumptions of Theorem~\ref{thm:first-order-supervision_a}, suppose the
first-order condition is satisfied up to error $\varepsilon\ge 0$, meaning
\[
\left\langle
\nabla_y E_\theta(x,y^\star),\, y-y^\star
\right\rangle
\ge
-\varepsilon
\qquad
\text{for all } y\in\mathcal Y.
\]
Then $y^\star$ is $\varepsilon$-globally optimal:
\[
E_\theta(x,y^\star)
\le
\min_{y\in\mathcal Y}E_\theta(x,y)+\varepsilon.
\]
\end{corollary}

\begin{proof}
By convexity, for every $y\in\mathcal Y$,
\[
E_\theta(x,y)
\ge
E_\theta(x,y^\star)
+
\left\langle
\nabla_y E_\theta(x,y^\star),\, y-y^\star
\right\rangle .
\]
Using the approximate first-order condition,
\[
E_\theta(x,y)
\ge
E_\theta(x,y^\star)-\varepsilon
\qquad
\text{for all } y\in\mathcal Y.
\]
Equivalently,
\[
E_\theta(x,y^\star)
\le
E_\theta(x,y)+\varepsilon
\qquad
\text{for all } y\in\mathcal Y.
\]
Taking the minimum over $y\in\mathcal Y$ gives
\[
E_\theta(x,y^\star)
\le
\min_{y\in\mathcal Y}E_\theta(x,y)+\varepsilon.
\]
\end{proof}

\begin{corollary}[Convergence of projected convex inference]
\label{cor:projected-convergence_a}
Let $E_\theta(x,\cdot)$ be convex and $L$-smooth on a closed convex set
$\mathcal Y$. Let $y^\star\in\arg\min_{y\in\mathcal Y}E_\theta(x,y)$, and run
projected gradient descent with stepsize $\eta=1/L$:
\[
y_{t+1}
=
\Pi_{\mathcal Y}
\left(
y_t-\frac{1}{L}\nabla_y E_\theta(x,y_t)
\right).
\]
Then after $T$ iterations,
\[
E_\theta(x,y_T)-E_\theta(x,y^\star)
\le
\frac{L\|y_0-y^\star\|_2^2}{2T}.
\]

If, in addition, $E_\theta(x,\cdot)$ is $\mu$-strongly convex on $\mathcal Y$,
then projected gradient descent converges linearly:
\[
\|y_T-y^\star\|_2^2
\le
\left(1-\frac{\mu}{L}\right)^T
\|y_0-y^\star\|_2^2 .
\]
\end{corollary}

\section{Ablations}
\label{app:extra}

\begin{figure*}[h!]
\centering

\begin{subfigure}{0.44\textwidth}
    \centering
    \includegraphics[width=\linewidth]{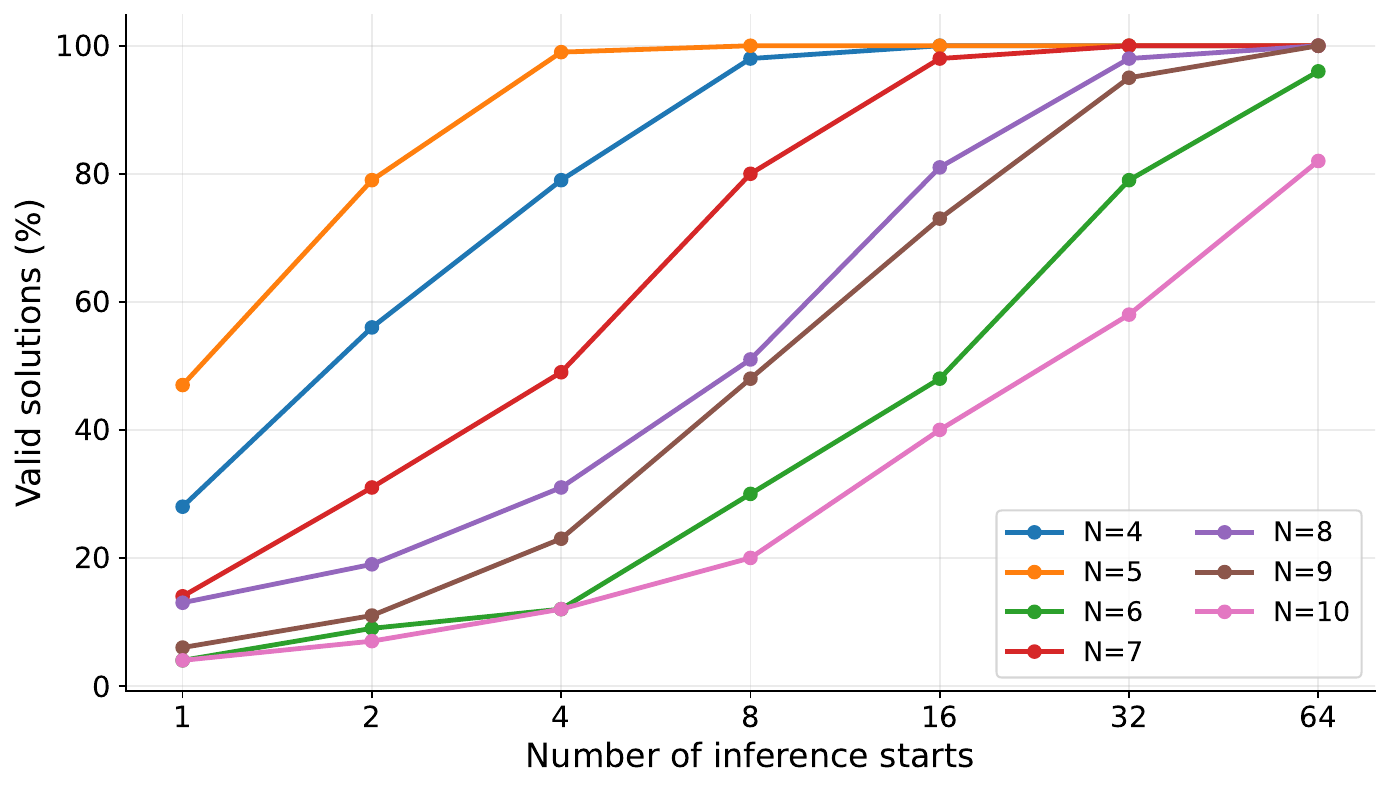}
    \caption{Valid solutions.}
\end{subfigure}
\begin{subfigure}{0.44\textwidth}
    \centering
    \includegraphics[width=\linewidth]{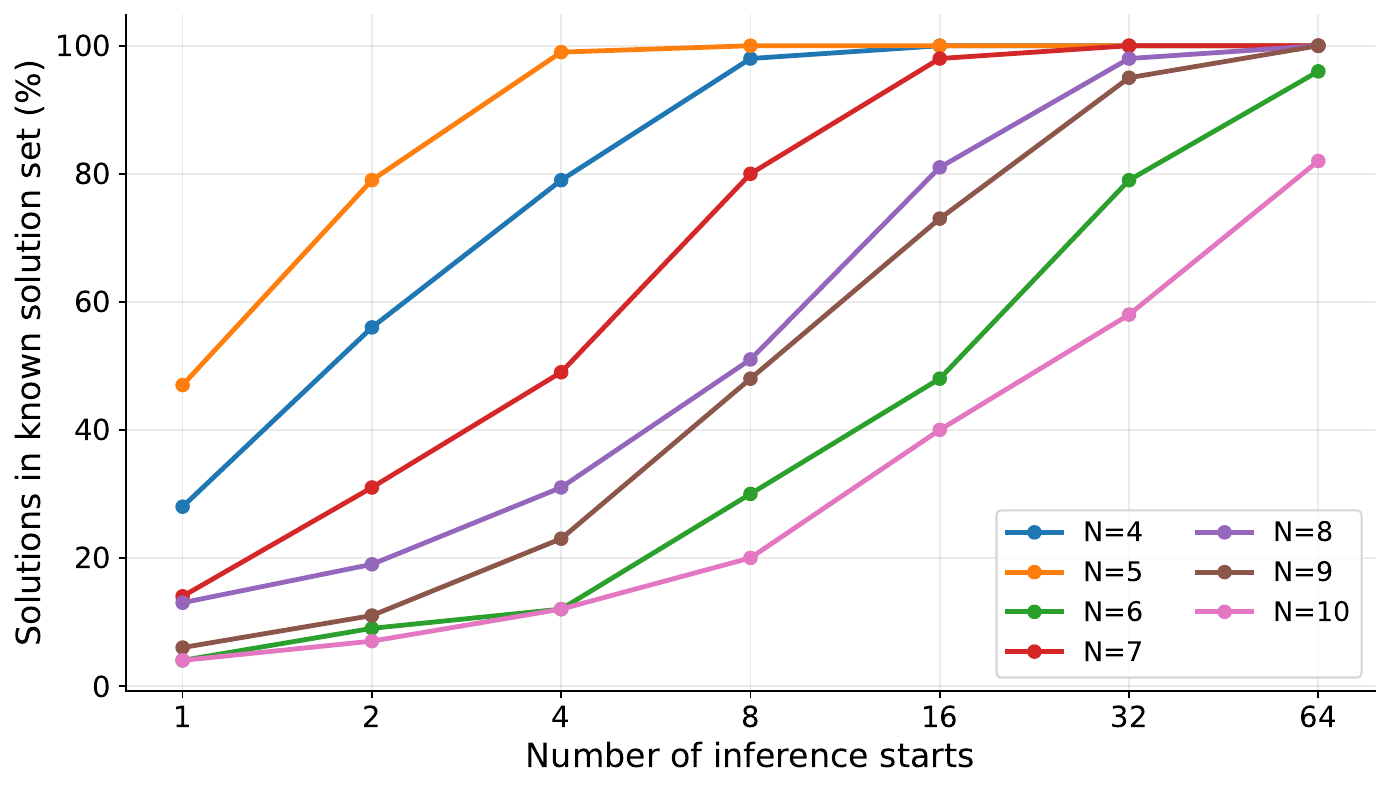}
    \caption{Known solution rate.}
\end{subfigure}

\vspace{0.6em}

\begin{subfigure}{0.44\textwidth}
    \centering
    \includegraphics[width=\linewidth]{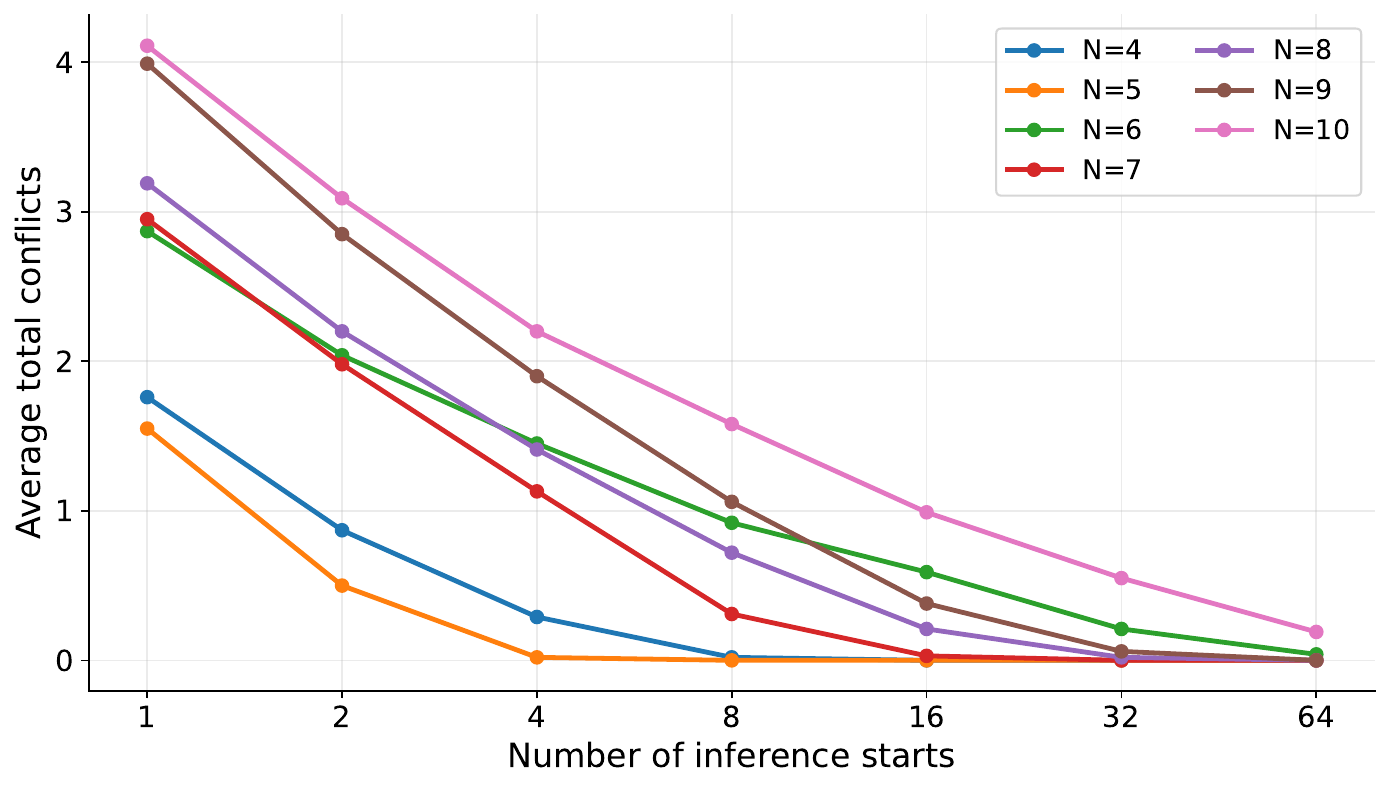}
    \caption{Average conflicts.}
\end{subfigure}
\begin{subfigure}{0.44\textwidth}
    \centering
    \includegraphics[width=\linewidth]{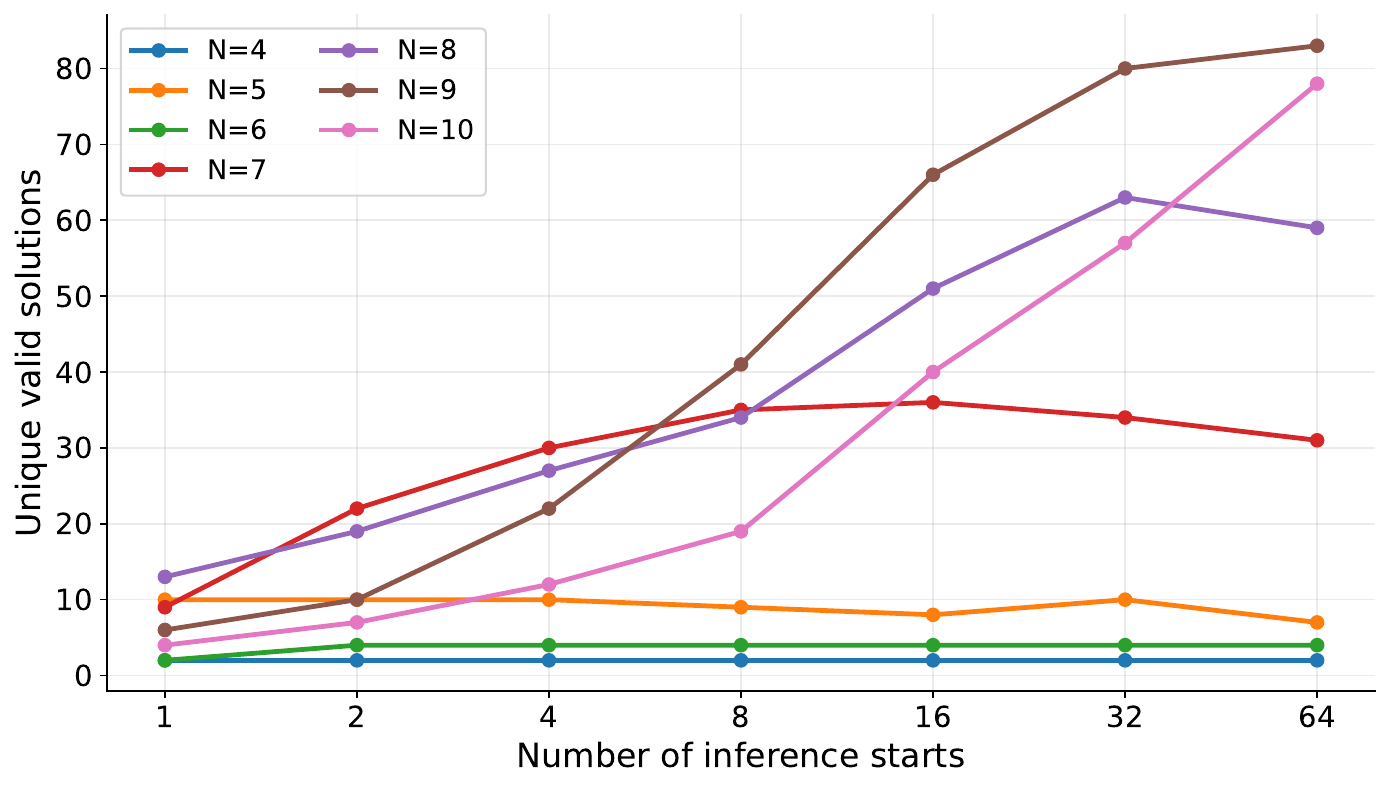}
    \caption{Unique valid solutions.}
\end{subfigure}

\caption{
Inference-start ablation for N-Queens. Increasing the number of independent
starts improves the probability of finding a valid decoded board, reduces the
average number of conflicts, and increases the diversity of valid solutions.
}
\label{fig:starts-ablation}
\end{figure*}

\begin{figure}[h!]
\centering

\begin{subfigure}{0.44\linewidth}
    \centering
    \includegraphics[width=\linewidth]{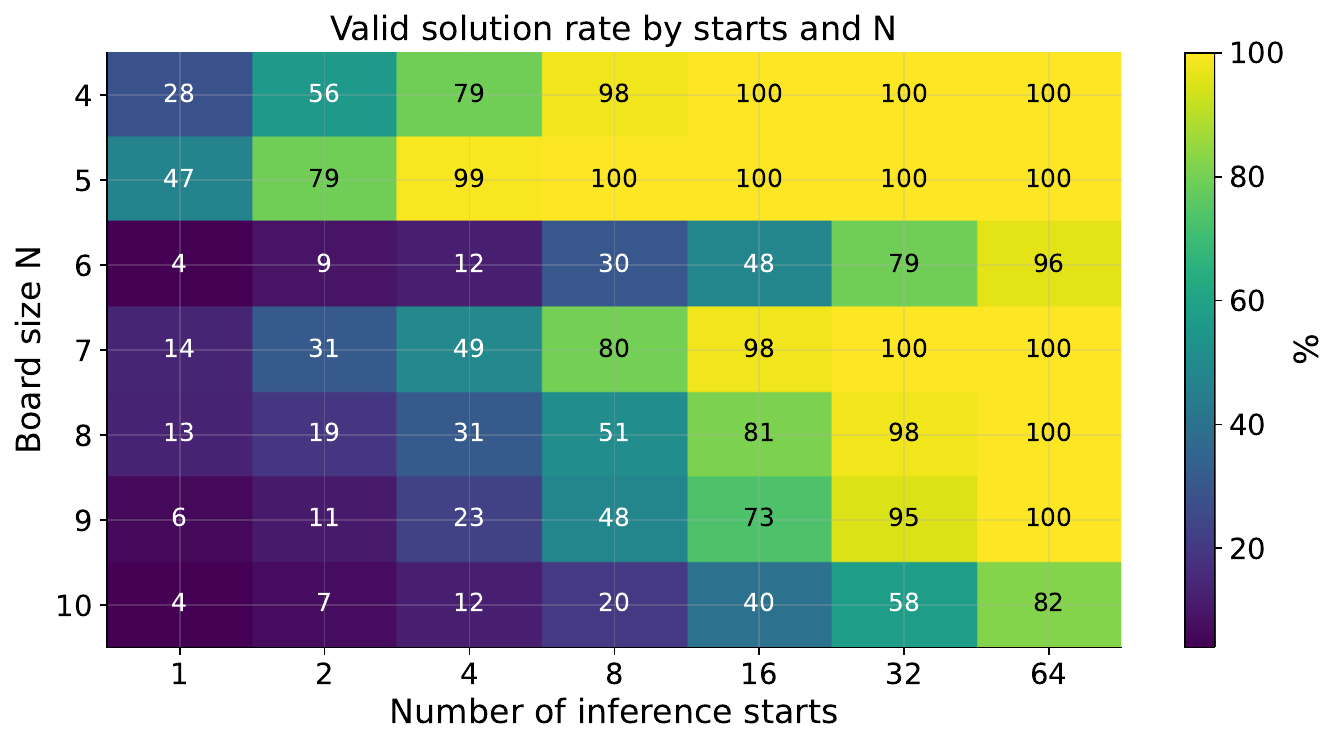}
    \caption{Valid solution rate.}
\end{subfigure}
\begin{subfigure}{0.44\linewidth}
    \centering
    \includegraphics[width=\linewidth]{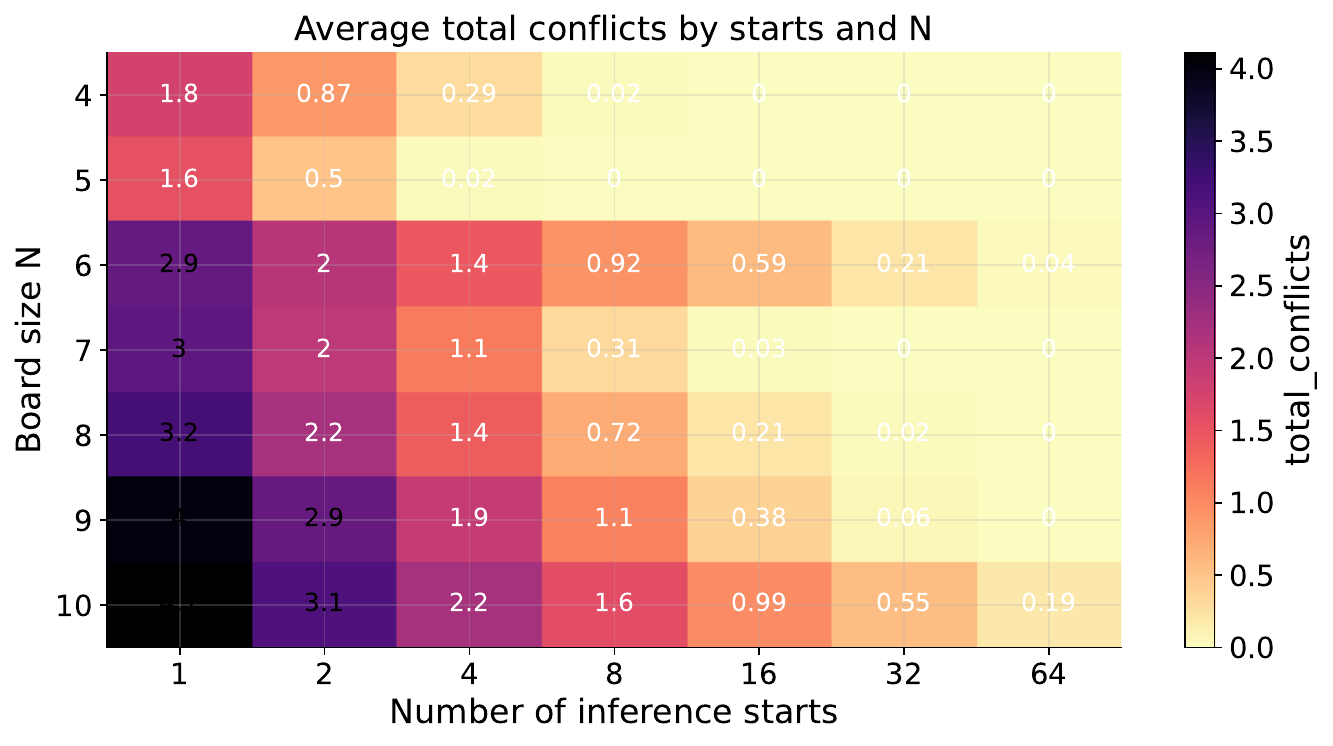}
    \caption{Average conflicts.}
\end{subfigure}

\caption{
Heatmap view of the inference-start ablation across board sizes. Larger boards
require more starts to obtain comparable validity, while increasing the number
of starts consistently reduces the average number of conflicts.
}
\label{fig:starts-ablation-heatmaps}
\end{figure}

\begin{table*}[t]
\centering
\scriptsize
\caption{N-Queens valid decoded boards (\%, mean $\pm$ standard deviation)
across inference starts. Each entry summarizes evaluations on 100 boards.}
\label{tab:nqueens-valid-errorbars}
\resizebox{\textwidth}{!}{%
\begin{tabular}{rccccccc}
\toprule
$N$ & $S=1$ & $S=2$ & $S=4$ & $S=8$ & $S=16$ & $S=32$ & $S=64$ \\
\midrule
4  & $28.00\pm3.29$ & $56.00\pm4.02$ & $79.00\pm3.68$ & $98.00\pm1.99$ & $100.00\pm0.58$ & $100.00\pm0.00$ & $100.00\pm0.00$ \\
5  & $47.00\pm5.69$ & $79.00\pm4.18$ & $99.00\pm2.20$ & $100.00\pm0.63$ & $100.00\pm0.00$ & $100.00\pm0.00$ & $100.00\pm0.00$ \\
6  & $4.00\pm1.79$ & $9.00\pm3.03$ & $12.00\pm4.31$ & $30.00\pm4.72$ & $48.00\pm4.47$ & $79.00\pm4.23$ & $96.00\pm1.97$ \\
7  & $14.00\pm3.74$ & $31.00\pm4.80$ & $49.00\pm5.43$ & $80.00\pm4.29$ & $98.00\pm1.97$ & $100.00\pm0.45$ & $100.00\pm0.00$ \\
8  & $13.00\pm2.38$ & $19.00\pm3.11$ & $31.00\pm4.70$ & $51.00\pm5.22$ & $81.00\pm3.87$ & $98.00\pm2.09$ & $100.00\pm0.39$ \\
9  & $6.00\pm2.31$ & $11.00\pm3.44$ & $23.00\pm4.45$ & $48.00\pm3.94$ & $73.00\pm4.47$ & $95.00\pm2.65$ & $100.00\pm0.79$ \\
10 & $4.00\pm1.82$ & $7.00\pm2.15$ & $12.00\pm3.12$ & $20.00\pm3.85$ & $40.00\pm4.88$ & $58.00\pm5.46$ & $82.00\pm3.62$ \\
\bottomrule
\end{tabular}}
\end{table*}

\begin{table*}[t]
\centering
\scriptsize
\caption{N-Queens average decoded conflicts (mean $\pm$ standard deviation)
across inference starts.}
\label{tab:nqueens-conflicts-errorbars}
\resizebox{\textwidth}{!}{%
\begin{tabular}{rccccccc}
\toprule
$N$ & $S=1$ & $S=2$ & $S=4$ & $S=8$ & $S=16$ & $S=32$ & $S=64$ \\
\midrule
4  & $1.80\pm0.13$ & $0.87\pm0.09$ & $0.29\pm0.05$ & $0.02\pm0.02$ & $0.00\pm0.01$ & $0.00\pm0.00$ & $0.00\pm0.00$ \\
5  & $1.60\pm0.18$ & $0.50\pm0.10$ & $0.02\pm0.05$ & $0.00\pm0.01$ & $0.00\pm0.00$ & $0.00\pm0.00$ & $0.00\pm0.00$ \\
6  & $2.90\pm0.11$ & $2.00\pm0.11$ & $1.40\pm0.10$ & $0.92\pm0.08$ & $0.59\pm0.05$ & $0.21\pm0.04$ & $0.04\pm0.02$ \\
7  & $3.00\pm0.19$ & $2.00\pm0.15$ & $1.10\pm0.13$ & $0.31\pm0.08$ & $0.03\pm0.03$ & $0.00\pm0.00$ & $0.00\pm0.00$ \\
8  & $3.20\pm0.20$ & $2.20\pm0.15$ & $1.40\pm0.12$ & $0.72\pm0.10$ & $0.21\pm0.05$ & $0.02\pm0.02$ & $0.00\pm0.00$ \\
9  & $4.00\pm0.17$ & $2.90\pm0.14$ & $1.90\pm0.13$ & $1.10\pm0.10$ & $0.38\pm0.08$ & $0.06\pm0.03$ & $0.00\pm0.01$ \\
10 & $4.10\pm0.18$ & $3.10\pm0.13$ & $2.20\pm0.11$ & $1.60\pm0.10$ & $0.99\pm0.08$ & $0.55\pm0.07$ & $0.19\pm0.04$ \\
\bottomrule
\end{tabular}}
\end{table*}

\section{Computational Details}
\label{sec:compute-comparison}
All runs use a single GPU with no distributed training. The end-to-end method
comparison below was measured on an NVIDIA RTX A6000; the loss-component
ablation in Table~\ref{tab:speed} was measured separately on an RTX A4000. The
main model is a three-layer PICNN with hidden width 256, followed by
composed-energy refinement through an unrolled projected solver. Phase~1 uses
1000 AdamW epochs and Phase~2 uses 300 epochs, with $S=4$ starts during
refinement and up to $S=64$ at evaluation.

\begin{table}[t]
\centering
\small
\caption{End-to-end wall-clock comparison on one NVIDIA RTX A6000
(mean $\pm$ standard deviation).}
\label{tab:wall-clock-comparison}
\begin{tabular}{llrr}
\toprule
Method & Task & Train (min) & Eval. (s) \\
\midrule
EBM+PEM & $N$-Queens & $258\pm15$ & $340\pm25$ \\
EBM+PEM & 3-SAT & $619\pm36$ & $2000\pm30$ \\
EBM+PEM & Graph coloring & $155\pm9$ & $180\pm15$ \\
CCEM & $N$-Queens & $115\pm14$ & $60\pm30$ \\
CCEM & 3-SAT & $61\pm4.6$ & $60\pm5$ \\
CCEM & Graph coloring & $15\pm2.3$ & $30\pm3$ \\
\bottomrule
\end{tabular}
\end{table}

A CCEM trajectory performs one projected gradient update per step and one
additional energy-only evaluation for final candidate selection. PEM also
evaluates energy at every step for importance weighting and resampling. Thus a
CCEM trajectory uses 51 energy forwards for $N$-Queens and 141 for graph
coloring or 3-SAT, versus 201 for EBM+PEM.

\begin{table}[t]
\centering
\small
\caption{Energy evaluations per trajectory. A gradient evaluation includes
one energy forward followed by differentiation.}
\label{tab:per-trajectory-compute}
\begin{tabular}{llrrr}
\toprule
Task & Method & $T$ & Energy-only & Total forwards \\
\midrule
$N$-Queens & CCEM & 50 & 1 & 51 \\
$N$-Queens & EBM+PEM & 100 & 101 & 201 \\
Graph coloring & CCEM & 140 & 1 & 141 \\
Graph coloring & EBM+PEM & 100 & 101 & 201 \\
3-SAT & CCEM & 140 & 1 & 141 \\
3-SAT & EBM+PEM & 100 & 101 & 201 \\
\bottomrule
\end{tabular}
\end{table}

\begin{table*}[t]
\centering
\small
\caption{Full inference-budget comparison. Counts include all starts or
particles.}
\label{tab:full-budget-compute}
\resizebox{\textwidth}{!}{%
\begin{tabular}{llrrrrr}
\toprule
Task & Method & $T$ & Starts $S$/particles $P$ & Energy-only & Gradient evals. & Total forwards \\
\midrule
$N$-Queens & CCEM & 50 & 64 & 64 & 3,200 & 3,264 \\
$N$-Queens & EBM+PEM & 100 & 1,024 & 103,424 & 102,400 & 205,824 \\
Graph coloring & CCEM & 140 & 64 & 64 & 8,960 & 9,024 \\
Graph coloring & EBM+PEM & 100 & 128 & 12,928 & 12,800 & 25,728 \\
3-SAT & CCEM & 140 & 64 & 64 & 8,960 & 9,024 \\
3-SAT & EBM+PEM & 100 & 1,024 & 103,424 & 102,400 & 205,824 \\
\bottomrule
\end{tabular}}
\end{table*}

The largest reported CCEM budget uses 64 independent trajectories, whereas
EBM+PEM uses 1,024 particles for $N$-Queens and 3-SAT and 128 for graph
coloring. CCEM therefore reduces both the work within a trajectory and the
number of parallel trajectories used for the reported solution quality.

Table~\ref{tab:speed} reports wall-clock runtimes for the N-Queens Phase 2 loss ablation on the same hardware. The timings show that adding the ranking and hard-negative terms introduces only a small overhead relative to the regression-only objective. The MSE-only configuration required 1:53:35, corresponding to 6.82 seconds per iteration, while the full objective with MSE, ranking, and hard-negative terms required 1:55:47, corresponding to 6.95 seconds per iteration. Thus, the full board-level objective increases per-iteration time by approximately $1.9\%$ compared with the MSE-only setting. The most expensive ablated configuration was the Rank+Hard objective without the MSE term, which took 2:00:07, or 7.21 seconds per iteration.

\begin{table}[t]
\centering
\caption{Runtime comparison across loss components for the N-Queens Phase 2 refinement experiment. Times are wall-clock durations on a single NVIDIA RTX A4000 GPU.}
\label{tab:speed}
\begin{tabular}{ccccc}
\toprule
MSE ($\alpha$) & Rank ($\beta$) & Hard ($\gamma$) & Time & Avg. sec/iter \\
\midrule
\(\checkmark\) &              &              & 1:53:35 & 6.82 \\
               & \(\checkmark\) & \(\checkmark\) & 2:00:07 & 7.21 \\
\(\checkmark\) & \(\checkmark\) &              & 1:54:26 & 6.87 \\
\(\checkmark\) & \(\checkmark\) & \(\checkmark\) & 1:55:47 & 6.95 \\
\bottomrule
\end{tabular}
\end{table}

The total wall-clock time for the four N-Queens ablation runs in Table~\ref{tab:speed} was approximately 7 hours and 44 minutes. These measurements are intended to characterize the cost of the proposed loss components under a fixed implementation and hardware setup, rather than to provide a fully hardware-independent benchmark. In practice, the dominant computational cost comes from backpropagating through the unrolled projected solver; the additional energy-margin and hard-negative terms require extra energy evaluations but do not change the asymptotic structure of the training loop. At inference time, compute scales linearly with the number of projected solver steps and the number of warm starts or particles used by the sampler.

\clearpage
\section{Additional experiments}
For Graph Coloring evaluation we used parameters from Table~\ref{tab:graph-coloring}.
\begin{table}[H]
\centering
\caption{Graph Coloring problem distribution parameters.}
\label{tab:graph-coloring}
\begin{tabular}{lcccc}
\toprule
Distribution & $V$ & $E$ & $d$ & $\chi$\\
\midrule
Erdos Renyi  & [20, 39] & [29, 76] & 0.12 & [3, 4] \\
Erdos Renyi 2 & [81, 99] & [193, 225] & 0.05 & [3, 4] \\
Holme Kim & [22, 34] & [56, 92] & 0.26 & [4, 4] \\
Holme Kim 2 & [86, 100] & [398, 469] & 0.10 & [5, 6]  \\
Regular Expander & [21, 40] & [63, 120] & 0.22 & [4, 4] \\
Regular Expander 2 & [86, 100] & [184, 200] & 0.23 & [3, 3] \\
Paley & [19, 37] & [171, 465] & 0.80 & [6, 10] \\
Complete & [8, 12] & [36, 66] & 1.00 & [8, 12]  \\
\bottomrule
\end{tabular}
\end{table}

\begin{table}[H]
\centering
\caption{Exact valid colorings after discrete decoding (out of 100). The
EBM+PEM values are from our evaluation of the authors' public implementation.}
\label{tab:gc-valid}
\begin{tabular}{lrr}
\toprule
Distribution & EBM+PEM & CCEM (ours) \\
\midrule
Erdos Renyi & 11 & \textbf{29} \\
Erdos Renyi 2 & 0 & \textbf{4} \\
Holme Kim & 0 & \textbf{3} \\
Holme Kim 2 & 0 & 0 \\
Regular Expander & 7 & \textbf{20} \\
Regular Expander 2 & 0 & 0 \\
Paley & 32 & \textbf{59} \\
Complete & 57 & \textbf{100} \\
\bottomrule
\end{tabular}
\end{table}

For 3-SAT, each clause has seven satisfying local assignments and one
falsifying assignment. Complementation is used only to simplify local-factor
pretraining; original literal polarities are preserved when clause factors are
composed and when the full formula is evaluated. This corrects the overly
compressed submitted description, which could be read as searching for an
assignment that violates every original clause and then negating it. Training
uses random 3-SAT instances with 10--20 variables near the phase transition
($4.258n$ clauses). Evaluation uses 100 formulas with 20 variables and 91
clauses.

\begin{table}[H]
\centering
\caption{3-SAT evaluation on 100 sampled formulas. Fully solved is a percentage
of formulas; clause error is reported in percentage points.}
\label{tab:sat-main}
\scriptsize
\begin{tabular}{llccc}
\toprule
Model & Type & Satisfied clauses & Fully solved (\%) & Clause error \\
\midrule
GCN & SL & $0.9617\pm0.0264$ & -- & -- \\
DGL & SL+TS & $0.9520\pm0.0330$ & -- & -- \\
DIFUSCO ($T=50$) & SL+S & $0.9734\pm0.0156$ & -- & -- \\
Fast T2T ($T_S=1,T_G=1$) & SL+S & $0.9749\pm0.0210$ & -- & -- \\
Fast T2T ($T_S=5,T_G=5$) & SL+S & $0.9760\pm0.0273$ & -- & -- \\
NeuroSAT ($T=50$) & SL+IR & $0.9661\pm0.0185$ & -- & -- \\
NeuroSAT ($T=500$) & SL+IR & $0.9742\pm0.0154$ & -- & -- \\
NSNet ($T=50$) & SL+BP & $0.9845\pm0.0272$ & -- & -- \\
NSNet ($T=500$) & SL+BP & $0.9856\pm0.0266$ & -- & -- \\
EBM ($P=1024$) & SL+PEM & $0.9985\pm0.0048$ & -- & -- \\
\midrule
CCEM ($S=1$) & SL+ProjAdam & $0.9953\pm0.0131$ & $66.04\pm0.98$ & $0.4673\pm0.0131$ \\
CCEM ($S=2$) & SL+ProjAdam & $0.9955\pm0.0096$ & $67.02\pm0.66$ & $0.4519\pm0.0096$ \\
CCEM ($S=4$) & SL+ProjAdam & $0.9956\pm0.0094$ & $67.54\pm0.65$ & $0.4435\pm0.0094$ \\
CCEM ($S=8$) & SL+ProjAdam & $0.9956\pm0.0115$ & $67.94\pm0.76$ & $0.4381\pm0.0115$ \\
CCEM ($S=16$) & SL+ProjAdam & $0.9957\pm0.0123$ & $68.40\pm0.72$ & $0.4315\pm0.0123$ \\
CCEM ($S=32$) & SL+ProjAdam & $0.9958\pm0.0151$ & $69.01\pm0.87$ & $0.4211\pm0.0151$ \\
CCEM ($S=64$) & SL+ProjAdam & $0.9959\pm0.0122$ & $69.65\pm0.73$ & $0.4108\pm0.0122$ \\
\bottomrule
\end{tabular}
\end{table}

\end{document}